\documentclass{article}

     \usepackage[final,nonatbib]{neurips_2019}

\usepackage[utf8]{inputenc} 
\usepackage[T1]{fontenc}    
\usepackage{hyperref}       
\usepackage{url}            
\usepackage{booktabs}       
\usepackage{amsfonts}       
\usepackage{nicefrac}       
\usepackage{microtype}      

\usepackage{amsthm}
\usepackage{mathtools}
\usepackage{amssymb}
\usepackage{bm}
\usepackage{tikz}
\usetikzlibrary{arrows,automata, calc, positioning}
\usepackage{float}
\usepackage{adjustbox}
\usepackage{wrapfig}
\usepackage{subcaption} 
\usepackage{cprotect}
\usepackage{etoc}
\usepackage{titlesec}
\usepackage{enumitem}

\tikzset{stage/.style = {draw,minimum width=15mm,minimum height=7mm},
      edgenode/.style = {font=\small,near start}}

\newtheorem{proposition}{Proposition}
\newtheorem{definition}{Definition}

\newtheorem{corollary}{Corollary}

\newtheorem{theorem}{Theorem}
\newtheorem{hypothesis}{Hypothesis}
\newtheorem{algorithm}{Algorithm}

\DeclareMathOperator*{\argmin}{\arg\!\min}

\newcommand{\indep}{\raisebox{0.05em}{\rotatebox[origin=c]{90}{$\models$}}}

\usepackage{cancel}

\title{Kernel Instrumental Variable Regression}

\author{
   Rahul ~Singh \\
  MIT Economics\\
  \texttt{rahul.singh@mit.edu} \\
  \And
    Maneesh ~Sahani \\
  Gatsby Unit, UCL \\
  \texttt{maneesh@gatsby.ucl.ac.uk} \\
  \And
  Arthur ~Gretton \\
  Gatsby Unit, UCL \\
  \texttt{arthur.gretton@gmail.com} \\
}

\begin{document}

\maketitle


\begin{abstract}
  Instrumental variable (IV) regression is a strategy for learning causal relationships in
observational data. If measurements of input $X$ and output $Y$ are confounded,
the causal relationship can nonetheless be identified if an instrumental variable $Z$
is available that influences $X$ directly, but is conditionally
independent of $Y$ given $X$ and the unmeasured confounder. The classic two-stage least squares algorithm (2SLS) simplifies the estimation problem by modeling all relationships as linear functions. We propose kernel instrumental variable regression (KIV), a nonparametric generalization of 2SLS, modeling relations among $X$, $Y$, and $Z$ as nonlinear functions in reproducing kernel Hilbert spaces (RKHSs). We prove the consistency of KIV under mild assumptions, and derive conditions under which   convergence  occurs at the minimax optimal rate for unconfounded, single-stage RKHS regression. In doing so, we obtain an efficient ratio between training sample sizes used in the algorithm's first and second stages. In experiments, KIV outperforms state of the art alternatives for nonparametric IV regression.
\end{abstract}

\section{Introduction}

Instrumental variable regression is a method in causal statistics
for estimating the counterfactual effect of input $X$ on output $Y$
using observational data \cite{stock2003retrospectives}. If measurements of $(X,Y)$ are confounded,
the causal relationship--also called the structural relationship--can nonetheless be identified if an instrumental variable $Z$ is
available, which is independent of $Y$ conditional on $X$ and the unmeasured confounder. Intuitively, $Z$ only influences $Y$ via $X$, identifying the counterfactual relationship of interest. 

Economists and epidemiologists use instrumental variables to overcome issues of strategic interaction, imperfect compliance, and selection bias.
The original application is demand estimation: supply cost shifters ($Z$) only influence sales ($Y$) via price ($X$), thereby identifying counterfactual demand even though prices reflect both supply and demand market forces \cite{wright1928tariff,blundell2012measuring}. Randomized assignment of a drug ($Z$) only influences patient health ($Y$) via actual consumption of the drug ($X$), identifying the counterfactual effect of the drug even in the scenario of imperfect compliance \cite{angrist1996identification}. Draft lottery number ($Z$) only influences lifetime
earnings ($Y$) via military service ($X$), identifying the
counterfactual effect of military service on earnings despite selection bias
in enlistment \cite{angrist1990lifetime}.

The two-stage least squares algorithm (2SLS), widely used in economics,
simplifies the IV estimation problem by assuming linear relationships: in \textit{stage 1}, perform linear regression to obtain the conditional
means $\bar{x}(z):=\mathbb{E}_{X|Z=z}(X)$; in \textit{stage 2},
linearly regress outputs $Y$ on these conditional means.
2SLS works well when the underlying assumptions hold. In practice, the relation between $Y$ and $X$ may not be linear, nor may be the relation between $X$
and $Z$.

In the present work, we introduce kernel instrumental variable regression
(KIV), an easily implemented nonlinear generalization of 2SLS (Sections~\ref{sec:framework} and~\ref{sec:problemAndAlgo}).\footnote{Code: \url{https://github.com/r4hu1-5in9h/KIV}} In
\textit{stage 1} we learn a conditional mean embedding,
which is the conditional expectation $\mu(z):=\mathbb{E}_{X|Z=z}\psi(X)$
of features $\psi$ which map $X$ to a reproducing kernel Hilbert
space (RKHS) \cite{song2009hilbert}. For a sufficiently rich RKHS, called a characteristic RKHS, the mean embedding of a random variable
is injective \cite{sriperumbudur2010relation}. It follows that the conditional mean embedding characterizes the full  distribution of $X$ conditioned on $Z$, and not just the conditional mean. We then implement \textit{stage 2} via kernel ridge regression of outputs $Y$ on these conditional mean embeddings, following the two-stage distribution regression approach described by \cite{szabo2015two,szabo2016learning}. As in our work, the inputs for \cite{szabo2015two,szabo2016learning} are distribution embeddings. Unlike our case, the earlier work uses unconditional embeddings computed from independent samples.

As a key contribution of our work, we provide consistency guarantees
for the KIV algorithm for an increasing number of training samples
in stages 1 and 2 (Section~\ref{sec:consistency}). To establish stage 1 convergence, we note that
the conditional mean embedding \cite{song2009hilbert} is the solution to a regression problem \cite{grunewalder2012conditional,grunewalder2012modelling,grunewalder2013smooth}, and thus equivalent to kernel dependency estimation \cite{ciliberto2016consistent,cortes2005general}. We prove that the kernel estimator of the conditional mean embedding (equivalently, the conditional expectation operator) converges in RKHS-norm, generalizing classic results by \cite{smale2005shannon,smale2007learning}.
We allow the conditional mean embedding RKHS to be infinite-dimensional, which presents specific challenges that we carefully address in our analysis. We also discuss previous approaches to establishing consistency in both finite-dimensional \cite{grunewalder2012modelling} and infinite-dimensional \cite{song2009hilbert,song2010nonparametric,fukumizu2013kernel,hefny2015supervised,ciliberto2016consistent} settings.

We embed the stage 1 rates into stage 2 to get end-to-end guarantees for the two-stage procedure, adapting 
\cite{caponnetto2007optimal,szabo2015two,szabo2016learning}. In particular, we provide
a ratio of stage 1 to stage 2 samples required for minimax optimal rates in
the second stage, where the ratio depends on the difficulty of
each stage. We anticipate that these proof strategies will apply generally
in two-stage regression settings. 




\section{Related work} 

Several approaches have been proposed to generalize 2SLS to the
nonlinear setting, which we will compare in our experiments (Section~\ref{sec:experiments}). A first generalization is via basis function approximation \cite{newey2003instrumental}, an approach called sieve IV, with uniform convergence rates in \cite{chen2018optimal}. The challenge in \cite{chen2018optimal} is how to define an appropriate finite dictionary of basis functions. In a second approach, \cite{carrasco2007linear,darolles2011nonparametric} implement stage 1 by computing the conditional distribution of the input $X$
given the instrument $Z$ using a ratio of Nadaraya-Watson density
estimates. Stage 2 is then ridge regression in the space
of square integrable functions. The overall algorithm has a finite sample consistency guarantee, assuming smoothness of the $(X,Z)$ joint density in stage 1 and the regression in stage 2 \cite{darolles2011nonparametric}. Unlike our bound, \cite{darolles2011nonparametric} make no claim about the optimality of the result. Importantly, stage 1 requires the
solution of a statistically challenging problem: conditional density
estimation. Moreover, analysis assumes the same number of training samples used in both stages. We will discuss this bound in more
detail in Appendix~\ref{sec:comparisonDarollesNonparametric}
(we suggest that the reader first cover Section~\ref{sec:consistency}).

Our work also relates to kernel and IV approaches to learning dynamical systems, known in machine learning as predictive state representation models (PSRs) \cite{boots2013hilbert,hefny2015supervised,downey2017predictive} and in econometrics as panel data models \cite{anderson1981estimation,arellano1991some}. In this setting, predictive states (expected future features given
history) are updated in light of new observations. The calculation of
the predictive states corresponds to stage 1 regression, and the
states are updated via stage 2 regression. In the kernel case, the
predictive states are expressed as conditional mean embeddings
\cite{boots2013hilbert}, as in our setting.  Performance of the kernel PSR method is guaranteed by a finite sample bound  \cite[Theorem 2]{hefny2015supervised}, however this bound is not minimax optimal. Whereas \cite{hefny2015supervised} assume an equal number of training samples in stages 1 and 2, we find that unequal numbers of training samples matter for minimax optimality. More importantly, the bound makes strong
smoothness assumptions on the inputs to the stage 1 and stage 2 regression
functions, rather than assuming smoothness of the regression
functions as we do. We show that the smoothness assumptions on
the inputs made in \cite{hefny2015supervised} do not hold in our
setting, and we obtain stronger end-to-end bounds under more realistic
conditions. We discuss the PSR bound in more detail
in Appendix~\ref{sec:comparisonHefnyNonparametric}.

Yet another recent approach is deep IV, which uses neural networks in both stages and permits learning even for complex high-dimensional data such as images \cite{hartford2017deep}. Like \cite{darolles2011nonparametric}, \cite{hartford2017deep} implement stage 1 by estimating a conditional density. Unlike \cite{darolles2011nonparametric}, \cite{hartford2017deep} use a mixture density network \cite[Section 5.6]{Bishop06}, i.e. a mixture model parametrized by
a neural network on the instrument $Z$. Stage 2 is neural network regression, trained using stochastic gradient descent (SGD). This presents a challenge: each step of SGD requires expectations using the stage 1 model, which are computed by drawing samples and averaging. An unbiased gradient estimate requires two independent sets of samples from the stage 1 model \cite[eq. 10]{hartford2017deep}, though a single set of samples may be used if an upper bound on the loss is optimized \cite[eq. 11]{hartford2017deep}. By contrast, our stage 1 outputs--conditional mean embeddings--have a closed form solution and exhibit
lower variance than sample averaging from a conditional density model. No theoretical guarantee on the consistency of the neural network
approach has been provided.



In the econometrics literature, a few key assumptions make learning a nonparametric IV model tractable. These include the completeness condition \cite{newey2003instrumental}: the structural relationship between $X$ and $Y$ can be identified only if the stage 1 conditional expectation is injective. Subsequent works impose additional stability and link assumptions \cite{blundell2007semi,chen2012estimation,chen2018optimal}: the conditional expectation of a function of $X$ given $Z$ is a smooth function of $Z$. We adapt these assumptions to our setting, replacing the completeness condition with the characteristic property \cite{sriperumbudur2010relation}, and replacing the stability and link assumptions with the concept of prior \cite{smale2007learning,caponnetto2007optimal}. We describe the characteristic and prior assumptions in more detail below. 

Extensive use of IV estimation in applied economic research has revealed a common pitfall: weak instrumental variables. A weak instrument satisfies Hypothesis 1 below, but the relationship between a weak instrument $Z$ and input $X$ is negligible; $Z$ is essentially irrelevant. In this case, IV estimation becomes highly erratic \cite{bound1995problems}. In \cite{staiger1997instrumental}, the authors formalize this phenomenon with local analysis. See \cite{murray2006avoiding,stock2002survey} for practical and theoretical overviews, respectively. We recommend that practitioners resist the temptation to use many weak instruments, and instead use few strong instruments such as those described in the introduction.



Finally, our analysis connects early work on the RKHS with recent developments in the RKHS literature. In \cite{nashed1974generalized}, the authors introduce the RKHS to solve known, ill-posed functional equations. In the present work, we introduce the RKHS to estimate the solution to an uncertain, ill-posed functional equation. In this sense, casting the IV problem in an RKHS framework is not only natural; it is in the original spirit of RKHS methods.
For a comprehensive review of existing work and recent advances in kernel mean embedding research, we recommend \cite{muandet2017kernel,gretton2018notes}.


\section{Problem setting and definitions}\label{sec:framework}



\textbf{Instrumental variable:} 
We begin by introducing our causal assumption about the instrument. This prior knowledge, described informally in the introduction, allows us to recover the counterfactual effect of $X$ on $Y$. Let $(\mathcal{X},\mathcal{B}_{\mathcal{X}})$, $(\mathcal{Y},\mathcal{B}_{\mathcal{Y}})$, and $(\mathcal{Z},\mathcal{B}_{\mathcal{Z}})$ be measurable spaces. Let $(X,Y,Z)$ be a random variable on $\mathcal{X}\times \mathcal{Y}\times \mathcal{Z}$ with distribution $\rho$. 
\begin{hypothesis}\label{iv}
Assume 
\begin{enumerate}
    \item $Y=h(X)+e$ and $\mathbb{E}[e|Z]=0$
    \item $\rho(x|z)$ is not constant in $z$
\end{enumerate}
\end{hypothesis}
We call $h$ the \textit{structural function} of interest. The error term $e$ is unmeasured, confounding noise. Hypothesis~\ref{iv}.1, known as the exclusion restriction, was introduced by \cite{newey2003instrumental} to the nonparametric IV literature for its tractability. 
 Other hypotheses are possible, although a very different approach is then needed \cite{imbens2009identification}. Hypothesis~\ref{iv}.2, known as the relevance condition, ensures that $Z$ is actually informative.  In Appendix~\ref{sec:simpleIVadditiveNoise}, we compare Hypothesis~\ref{iv} with alternative formulations of the IV assumption.

We make three observations. First, if $X=Z$ then Hypothesis~\ref{iv} reduces to the standard regression assumption of unconfounded inputs, and $h(X)=\mathbb{E}[Y|X]$; if $X=Z$ then prediction and counterfactual prediction coincide. The IV model is a framework that allows for causal inference in a more general variety of contexts, namely when $h(X)\neq \mathbb{E}[Y|X]$ so that prediction and counterfactual prediction are different learning problems. Second, Hypothesis~\ref{iv} will permit identification of $h$ even if inputs are confounded, i.e. $X\cancel\indep e$. Third, this model includes the scenario in which the analyst has a combination of confounded and unconfounded inputs. For example, in demand estimation there may be confounded price $P$, unconfounded characteristics $W$, and supply cost shifter $C$ that instruments for price. Then $X=(P,W)$, $Z=(C,W)$, and the analysis remains the same. 

\begin{wrapfigure}{R}{0.25\textwidth}
\vspace{-15pt}
\begin{center}
\begin{adjustbox}{width=.25\textwidth}
\begin{tikzpicture}[->,>=stealth',shorten >=1pt,auto,node distance=2.8cm,
                    semithick]
  \tikzstyle{every state}=[draw=black,text=black]

  \node[state]         (z) [fill=gray]                   {$\mathcal{Z}$};
    \node[state]         (hz) [right of=z]       {$\mathcal{H}_{\mathcal{Z}}$};
  \node[state]         (x) [below of=z, fill=gray]       {$\mathcal{X}$};
  \node[state]         (hx) [right of=x]                  {$\mathcal{H}_{\mathcal{X}}$};
  \node[state]         (y) [below of=x, fill=gray]       {$\mathcal{Y}$};

  \path (z) edge              node {$ $} (x)
            edge             node {$\mu\in \mathcal{H}_{\Xi}$} (hx)
            edge             node {$\phi$} (hz)
        (x) edge              node {$h\in\mathcal{H}_{\mathcal{X}}$} (y)
            edge             node {$\psi$} (hx)
        (hz) edge           node {$E^*\in\mathcal{H}_{\Gamma ^*}$} (hx)
        (hx) edge           node {$H\in\mathcal{H}_{\Omega}$} (y);;
\end{tikzpicture}
\end{adjustbox}
\vspace{-10pt}
\caption{The RKHSs}
\label{rkhs_s}
\end{center}
\vspace{-15pt}
\end{wrapfigure}
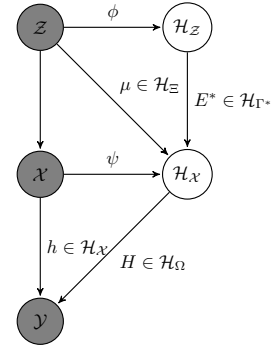

Hypothesis~\ref{iv} provides the operator equation $
\mathbb{E}[Y|Z]=\mathbb{E}_{X|Z}h(X)
$ \cite{newey2003instrumental}. In the language of 2SLS, the LHS is the \textit{reduced form}, while the RHS is a composition of \textit{stage 1} linear compact operator $\mathbb{E}_{X|Z}$ and \textit{stage 2} structural function $h$. In the language of functional analysis, the operator equation is a Fredholm integral equation of the first kind \cite{nashed1974generalized,kress1989linear,newey2003instrumental,florens2003inverse}. Solving this operator equation for $h$ involves inverting a linear compact operator with infinite-dimensional domain; it is an ill-posed problem \cite{kress1989linear}. To recover a well-posed problem, we impose smoothness and Tikhonov regularization.


\textbf{RKHS model:} 
We next introduce our RKHS model. Let $k_{\mathcal{X}}:\mathcal{X}\times \mathcal{X}\rightarrow\mathbb{R}$ and $k_{\mathcal{Z}}:\mathcal{Z}\times \mathcal{Z}\rightarrow\mathbb{R}$ be measurable positive definite kernels corresponding to scalar-valued RKHSs $\mathcal{H}_{\mathcal{X}}$ and $\mathcal{H}_{\mathcal{Z}}$. Denote the feature maps
\begin{align*}
\psi&:\mathcal{X}\rightarrow\mathcal{H}_{\mathcal{X}} ,\enskip x\mapsto k_{\mathcal{X}}(x,\cdot)\quad\quad \phi:\mathcal{Z}\rightarrow\mathcal{H}_{\mathcal{Z}}  ,\enskip z\mapsto k_{\mathcal{Z}}(z,\cdot)
\end{align*}

Define the \textit{conditional expectation operator} $E:\mathcal{H}_{\mathcal{X}}\rightarrow \mathcal{H}_{\mathcal{Z}}$ such that
$
[Eh](z)=\mathbb{E}_{X|Z=z}h(X)
$. $E$ is the natural object of interest for stage 1. We define and analyze an estimator for $E$ directly. The conditional expectation operator $E$ conveys exactly the same information as another object popular in the kernel methods literature, the \textit{conditional mean embedding} $\mu:\mathcal{Z}\rightarrow \mathcal{H}_{\mathcal{X}}$ defined by
$
\mu(z)=\mathbb{E}_{X|Z=z}\psi(X)
$ \cite{song2009hilbert}. Indeed, $\mu(z)=E^*\phi(z)$ where $E^*:\mathcal{H}_{\mathcal{Z}}\rightarrow\mathcal{H}_{\mathcal{X}}$ is the adjoint of $E$. Analogously, in 2SLS $\bar{x}(z)=\pi' z$ for stage 1 linear regression parameter $\pi$.

The structural function $h:\mathcal{X}\rightarrow \mathcal{Y}$ in Hypothesis \ref{iv} is the natural object of interest for stage 2. For theoretical purposes, it is convenient to estimate $h$ indirectly. The structural function $h$ conveys exactly the same information as an object we call the \textit{structural operator} $H:\mathcal{H}_{\mathcal{X}}\rightarrow\mathcal{Y}$. Indeed, $h(x)=H\psi(x)$. Analogously, in 2SLS $h(x)=\beta' x$ for structural parameter $\beta$. We define and analyze an estimator for $H$, which in turn implies an estimator for $h$. Figure~\ref{rkhs_s} summarizes the relationships among equivalent stage 1 objects $(E,\mu)$ and equivalent stage 2 objects $(H,h)$.



Our RKHS model for the IV problem is of the same form as the  model in \cite{nashed1974convergence,nashed1974generalized, nashed1974regularization} for general operator equations.
We begin by choosing RKHSs for the  structural function $h$ and the reduced form $\mathbb{E}[Y|Z]$, then
 construct a tensor-product RKHS for the conditional expectation operator $E$.
 Our model differs from the RKHS model proposed by \cite{carrasco2007linear,darolles2011nonparametric}, which directly learns the conditional expectation operator $E$ via Nadaraya-Watson density estimation. The RKHSs of \cite{engl1996regularization,carrasco2007linear,darolles2011nonparametric} for the structural function $h$ and the reduced form $\mathbb{E}[Y|Z]$
are defined from the right and left singular functions of $E$, respectively. They appear in the consistency argument, but not in the ridge penalty.

\section{Learning problem and algorithm}\label{sec:problemAndAlgo}

2SLS consists of two stages that can be estimated separately.  Sample splitting in this context means estimating stage 1 with $n$ randomly chosen observations and estimating stage 2 with the remaining $m$ observations. Sample splitting alleviates the finite sample bias of 2SLS when instrument $Z$ weakly influences input $X$ \cite{angrist1995split}. It is the natural approach when an analyst does not have access to a single data set with $n+m$ observations of $(X,Y,Z)$ but rather two data sets: $n$ observations of $(X,Z)$, and $m$ observations of $(Y,Z)$. We employ sample splitting in KIV, with an efficient ratio of $(n,m)$ given in Theorem~\ref{rate}. In our presentation of the general two-stage learning problem, we denote stage 1 observations by $(x_i,z_i)$ and stage 2 observations by $(\tilde{y}_i,\tilde{z}_i)$.

\subsection{Stage 1}


We transform the problem of learning $E$ into a vector-valued kernel ridge regression following \cite{grunewalder2012conditional,grunewalder2013smooth,ciliberto2016consistent}, where the hypothesis space is the vector-valued RKHS $\mathcal{H}_{\Gamma}$ of operators mapping $\mathcal{H}_{\mathcal{X}}$ to $\mathcal{H}_{\mathcal{Z}}$.
In Appendix \ref{sec:vectorValuedRKHS}, we review the theory of vector-valued RKHSs as it relates to scalar-valued RKHSs and tensor product spaces. The key result is that the tensor product space of $\mathcal{H}_{\mathcal{X}}$ and $\mathcal{H}_{\mathcal{Z}}$ is isomorphic to $\mathcal{L}_2(\mathcal{H}_{\mathcal{X}},\mathcal{H}_{\mathcal{Z}})$, the space of Hilbert-Schmidt operators from $\mathcal{H}_{\mathcal{X}}$ to $\mathcal{H}_{\mathcal{Z}}$.
If we choose the vector-valued kernel $\Gamma$ with feature map $(x,z)\mapsto [\phi(z)\otimes \psi(x)](\cdot)=\phi(z)\langle \psi(x),\cdot \rangle_{\mathcal{H}_{\mathcal{X}}}$, then $\mathcal{H}_{\Gamma}= \mathcal{L}_2(\mathcal{H}_{\mathcal{X}},\mathcal{H}_{\mathcal{Z}})$ and it shares the same norm. 


We now state the objective for optimizing $E\in \mathcal{H}_{\Gamma}$. 
The optimal $E$ minimizes the expected discrepancy 
\begin{align*}
    E_{\rho}&=\argmin \mathcal{E}_1(E),\quad \mathcal{E}_1(E)= \mathbb{E}_{(X,Z)} \|\psi(X)-E^*\phi(Z)\|^2_{\mathcal{H}_{\mathcal{X}}}
\end{align*}
Both \cite{grunewalder2013smooth} and \cite{ciliberto2016consistent} refer to $\mathcal{E}_1$ as the surrogate risk.
As shown in \cite[Section 3.1]{grunewalder2012conditional} and \cite{grunewalder2013smooth}, the surrogate risk upper bounds the natural risk for the conditional expectation,
where the bound becomes tight when $\mathbb{E}_{X|Z=(\cdot)}f(X)\in \mathcal{H}_{\mathcal{Z}},\:\forall f\in\mathcal{H}_{\mathcal{X}}$.
Formally, the target operator is the constrained solution $
    E_{\mathcal{H}_{\Gamma}}=\argmin_{E\in\mathcal{H}_{\Gamma}}\mathcal{E}_1(E)
$.
We will assume $E_{\rho}\in \mathcal{H}_{\Gamma}$ so that $E_{\rho}=E_{\mathcal{H}_{\Gamma}}$. 

Next we impose Tikhonov regularization. The regularized target operator and its empirical analogue are given by
\begin{align*}
E_{\lambda}&=\argmin_{E\in \mathcal{H}_{\Gamma}} \mathcal{E}_{\lambda}(E),\quad
   \mathcal{E}_{\lambda}(E)=\mathcal{E}_1(E)+\lambda\|E\|^2_{\mathcal{L}_2(\mathcal{H}_{\mathcal{X}},\mathcal{H}_{\mathcal{Z}})}  \\
   E^n_{\lambda}&=\argmin_{E\in \mathcal{H}_{\Gamma}} \mathcal{E}_{\lambda}^n(E),\quad 
    \mathcal{E}_{\lambda}^n(E)=\dfrac{1}{n}\sum_{i=1}^n \|\psi(x_i)-E^*\phi(z_i)\|^2_{\mathcal{H}_{\mathcal{X}}}+\lambda\|E\|^2_{\mathcal{L}_2(\mathcal{H}_{\mathcal{X}},\mathcal{H}_{\mathcal{Z}})}
\end{align*}
Our construction of a vector-valued RKHS $\mathcal{H}_{\Gamma}$ for the conditional expectation operator $E$ permits us to estimate stage 1 by kernel ridge regression. The stage 1 estimator of KIV is at once novel in the nonparametric IV literature and fundamentally similar to 2SLS. Basis function approximation \cite{newey2003instrumental,chen2018optimal} is perhaps the closest prior IV approach, but we use infinite dictionaries of basis functions $\psi$ and $\phi$. Compared to density estimation \cite{carrasco2007linear, darolles2011nonparametric,hartford2017deep}, kernel ridge regression is an easier problem.

Alternative stage 1 estimators in the literature estimate the singular system of $E$ to ensure that the adjoint of the estimator equals the estimator of the adjoint. These estimators differ in how they estimate the singular system: empirical distribution \cite{darolles2011nonparametric}, Nadaraya-Watson density \cite{darolles2004kernel}, or B-spline wavelets \cite{chen1997shape}. The KIV stage 1 estimator has the desired property by construction; $(E_{\lambda}^n)^*=(E^*)_{\lambda}^n$. See Appendix~\ref{sec:vectorValuedRKHS} for details.

\subsection{Stage 2}

Next, we transform the problem of learning $h$ into a scalar-valued kernel ridge regression that respects the  IV problem structure. In Proposition \ref{prop:existenceOfConditionalMeanElement} of Appendix~\ref{sec:vectorValuedRKHS}, we show that under Hypothesis \ref{hyp:measurableBoundedFeatures} below,
$$
 \mathbb{E}_{X|Z=z}h(X)=[Eh](z)=\langle h,\mu(z)\rangle_{\mathcal{H}_{\mathcal{X}}}=H\mu(z)
$$
 where $h\in\mathcal{H}_{\mathcal{X}}$, a scalar-valued RKHS; $E\in \mathcal{H}_{\Gamma}$, the vector-valued RKHS described above; $\mu\in\mathcal{H}_{\Xi}$, a vector-valued RKHS isometrically isomorphic to $\mathcal{H}_{\Gamma}$; and $H\in \mathcal{H}_{\Omega}$, a scalar-valued RKHS isometrically isomorphic to $\mathcal{H}_{\mathcal{X}}$.
 It is helpful to think of $\mu(z)$ as the embedding into $\mathcal{H}_{\mathcal{X}}$ of a distribution on $\mathcal{X}$ indexed by the conditioned value $z$. When  $k_\mathcal{X}$ is characteristic, $\mu(z)$ uniquely embeds the conditional distribution, and  $H$ is identified. The kernel $\Omega$ satisfies $k_{\mathcal{X}}(x,x')=\Omega(\psi(x),\psi(x'))$.  This expression establishes the formal connection between our model and \cite{szabo2015two,szabo2016learning}. The choice of $\Omega$ may be more general; for nonlinear examples see \cite[Table 1]{szabo2016learning}.

We now state the objective for optimizing $H\in \mathcal{H}_{\Omega}$. Hypothesis~\ref{iv} provides the operator equation, which may be rewritten as the regression equation
$$
Y=\mathbb{E}_{X|Z}h(X)+e_Z=H\mu(Z)+e_Z,\quad \mathbb{E}[e_Z|Z]=0
$$
The unconstrained solution is
\begin{align*}
    H_{\rho}&=\argmin \mathcal{E}(H),\quad 
    \mathcal{E}(H) = \mathbb{E}_{(Y,Z)}\|Y-H\mu(Z)\|_{\mathcal{Y}}^2 
\end{align*}
The target operator is the constrained solution $
    H_{\mathcal{H}_{\Omega}}=\argmin_{H\in\mathcal{H}_{\Omega}}\mathcal{E}(H)
$. We will assume $H_{\rho}\in \mathcal{H}_{\Omega}$ so that $H_{\rho}=H_{\mathcal{H}_{\Omega}}$. With regularization,
\begin{align*}
    H_{\xi}&=\argmin_{H\in \mathcal{H}_{\Omega}}\mathcal{E}_{\xi}(H),\quad 
    \mathcal{E}_{\xi}(H)=\mathcal{E}(H)+\xi\|H\|^2_{\mathcal{H}_{\Omega}} \\
    H^{m}_{\xi}&=\argmin_{H\in \mathcal{H}_{\Omega}}\mathcal{E}^{m}_{\xi}(H),\quad 
    \mathcal{E}^{m}_{\xi}(H)=\dfrac{1}{m}\sum_{i=1}^{m}\|\tilde{y}_i-H\mu(\tilde{z}_i)\|_{\mathcal{Y}}^2+\xi\|H\|^2_{\mathcal{H}_{\Omega}}
\end{align*}

The essence of the IV problem is this: we do not directly observe the conditional expectation operator $E$ (or equivalently the conditional mean embedding $\mu$) that appears in the stage 2 objective. Rather, we approximate it using the estimate from stage 1. Thus our KIV estimator is $\hat{h}^{m}_{\xi}=\hat{H}^{m}_{\xi}\psi$ where
\begin{align*}
    \hat{H}^{m}_{\xi}&=\argmin_{H\in \mathcal{H}_{\Omega}}\hat{\mathcal{E}}^{m}_{\xi}(H),\quad 
    \hat{\mathcal{E}}^{m}_{\xi}(H)=\dfrac{1}{m}\sum_{i=1}^{m}\|\tilde{y}_i-H\mu^n_{\lambda}(\tilde{z}_i) \|_{\mathcal{Y}}^2+\xi\|H\|^2_{\mathcal{H}_{\Omega}}
\end{align*}
and $\mu^n_{\lambda}=(E_{\lambda}^n)^*\phi$. The transition from $H_{\rho}$ to $H^{m}_{\xi}$ represents the fact that we only have $m$ samples. The transition from $H^{m}_{\xi}$ to $\hat{H}^{m}_{\xi}$ represents the fact that we must learn not only the structural operator $H$ but also the conditional expectation operator $E$. In this sense, the IV problem is more complex than the estimation problem considered by \cite{nashed1974convergence,nashed1974regularization} in which $E$ is known.


\subsection{Algorithm}
We obtain a closed form expression for the KIV estimator. The apparatus introduced above is required for analysis of consistency and convergence rate. More subtly, our RKHS construction allows us to write kernel ridge regression estimators for both stage 1 and stage 2, unlike previous work. Because KIV consists of repeated kernel ridge regressions, it benefits from repeated applications of the representer theorem \cite{wahba1990spline,scholkopf2001generalized}. Consequently, we have a shortcut for obtaining KIV's closed form; see Appendix~\ref{sec:alg_deriv} for the full derivation.
\begin{algorithm}\label{alg} Let $X$ and $Z$ be matrices of $n$ observations. Let $\tilde{y}$ and $\tilde{Z}$ be a vector and matrix of $m$ observations.
\begin{align*}
    W&=K_{XX}(K_{ZZ}+n\lambda I)^{-1}K_{Z\tilde{Z}},\quad 
    \hat{\alpha}= (WW'+m\xi K_{XX})^{-1}W\tilde{y},\quad 
    \hat{h}_{\xi}^m(x)=(\hat{\alpha})'K_{Xx}
\end{align*}
where $K_{XX}$ and $K_{ZZ}$ are the empirical kernel matrices.
\end{algorithm}
Theorems~\ref{stage1} and~\ref{rate} below theoretically determine efficient rates for the stage 1 regularization parameter $\lambda$ and stage 2 regularization parameter $\xi$, respectively. In Appendix~\ref{sec:validation}, we provide a validation procedure to empirically determine values for $(\lambda,\xi)$.

\section{Consistency}\label{sec:consistency}

\subsection{Stage 1}

\textbf{Integral operators:} We use integral operator notation from the kernel methods literature, adapted to the conditional expectation operator learning problem. We denote by $L^2(\mathcal{Z},\rho_{\mathcal{Z}})$ the space of square integrable functions from $\mathcal{Z}$ to $\mathcal{Y}$ 
with respect to measure $\rho_{\mathcal{Z}}$, where $\rho_{\mathcal{Z}}$ is the restriction of $\rho$ to $\mathcal{Z}$.
\begin{definition}
The stage 1 (population) operators are
\begin{align*}
S_1^*&:\mathcal{H}_{\mathcal{Z}}\hookrightarrow L^2(\mathcal{Z},\rho_{\mathcal{Z}}) ,\enskip \ell\mapsto \langle \ell,\phi(\cdot) \rangle_{\mathcal{H}_{\mathcal{Z}}}\quad 
S_1:L^2(\mathcal{Z},\rho_{\mathcal{Z}})\rightarrow\mathcal{H}_{\mathcal{Z}},\enskip \tilde{\ell}\mapsto \int\phi(z)\tilde{\ell}(z)d\rho_{\mathcal{Z}}(z)
\end{align*}
\end{definition}
$T_1=S_1\circ S_1^*$ is the uncentered covariance operator of \cite[Theorem 1]{fukumizu2004dimensionality}. In Appendix~\ref{sec:cov_technical}, we prove that $T_1$ exists and has finite trace even when $\mathcal{H}_{\mathcal{X}}$ and $\mathcal{H}_{\mathcal{Z}}$ are infinite-dimensional. In Appendix~\ref{sec:cov_review}, we compare $T_1$ with other covariance operators in the kernel methods literature.

\textbf{Assumptions:} We place  assumptions on the original spaces $\mathcal{X}$ and $\mathcal{Z}$, the scalar-valued RKHSs $\mathcal{H}_{\mathcal{X}}$ and $\mathcal{H}_{\mathcal{Z}}$, and the probability distribution $\rho(x,z)$. We maintain these assumptions throughout the paper. Importantly, we assume that the vector-valued RKHS regression is correctly specified: the true conditional expectation operator $E_{\rho}$ lives in the vector-valued RKHS $\mathcal{H}_{\Gamma}$. In further research, we will relax this assumption.

\begin{hypothesis}Suppose that $\mathcal{X}$ and $\mathcal{Z}$ are Polish spaces, i.e. separable and completely metrizable topological spaces
\end{hypothesis}
\begin{hypothesis}\label{hyp:measurableBoundedFeatures} Suppose that
\begin{enumerate}

    \item $k_{\mathcal{X}}$ and $k_{\mathcal{Z}}$ are continuous and bounded:
    $
        \sup_{x\in\mathcal{X}}\|\psi(x)\|_{\mathcal{H}_{\mathcal{X}}}\leq Q$, $
        \sup_{z\in\mathcal{Z}}\|\phi(z)\|_{\mathcal{H}_{\mathcal{Z}}}\leq \kappa
  $
    \item $\psi$ and $\phi$ are measurable
    \item $k_{\mathcal{X}}$ is characteristic \textup{\cite{sriperumbudur2010relation}}
\end{enumerate}
\end{hypothesis}
\begin{hypothesis}Suppose that $E_{\rho}\in \mathcal{H}_{\Gamma}$. Then $
        \mathcal{E}_1(E_{\rho})=\inf_{E\in\mathcal{H}_{\Gamma}}\mathcal{E}_1(E)
    $
\end{hypothesis}
Hypothesis 3.3 specializes the completeness condition of \cite{newey2003instrumental}. Hypotheses 2-4 are sufficient to bound the sampling error of the regularized estimator $E_{\lambda}^n$. Bounding the approximation error requires a further assumption on the smoothness of the distribution $\rho(x,z)$.
We assume $\rho(x,z)$ belongs to a class of distributions parametrized by $(\zeta_1,c_1)$, as generalized from \cite[Theorem 2]{smale2007learning} to the space $\mathcal{H}_\Gamma$.
\begin{hypothesis}\label{hyp:smaleSmoothness}
Fix $\zeta_1<\infty$. For given $c_1\in(1,2]$, define the prior $\mathcal{P}(\zeta_1,c_1)$ as the set of probability distributions $\rho$ on $\mathcal{X}\times \mathcal{Z}$ such that a range space assumption is satisfied: $\exists G_1\in\mathcal{H}_{\Gamma}$ s.t.
   $
    E_{\rho}=T_1^{\frac{c_1-1}{2}}\circ G_1$ and $ 
    \|G_1\|^2_{\mathcal{H}_{\Gamma}}\leq \zeta_1
   $
\end{hypothesis}
We use composition symbol $\circ$ to emphasize that $G_1:\mathcal{H}_{\mathcal{X}}\rightarrow\mathcal{H}_{\mathcal{Z}}$ and $T_1:\mathcal{H}_{\mathcal{Z}}\rightarrow \mathcal{H}_{\mathcal{Z}}$. We define the power of operator $T_1$ with respect to its eigendecomposition; see Appendix~\ref{sec:cov_technical} for formal justification. Larger $c_1$ corresponds to a smoother conditional expectation operator $E_{\rho}$. Proposition~\ref{cef} in Appendix~\ref{sec:stage1_reg} shows $E_{\rho}^*\phi(z)=\mu(z)$, so Hypothesis 5 is an indirect smoothness condition on the conditional mean embedding $\mu$.





\textbf{Estimation and convergence:} The estimator has a closed form solution, as noted in \cite[Section 3.1]{grunewalder2012conditional} and \cite[Appendix D]{grunewalder2012modelling}; \cite{ciliberto2016consistent} use it in the first stage of the structured prediction problem. We present the closed form solution in notation similar to \cite{caponnetto2007optimal} in order to elucidate how the estimator simply generalizes linear regression. This connection foreshadows our proof technique.
\begin{theorem}\label{sol_1}
$\forall \lambda>0$, the solution $E^n_{\lambda}$ of the regularized empirical objective $\mathcal{E}^n_{\lambda}$ exists, is unique, and
$$
E^n_{\lambda}= (\mathbf{T}_1+\lambda)^{-1}\circ\mathbf{g}_1,\quad \mathbf{T}_1=\dfrac{1}{n}\sum_{i=1}^n \phi(z_i)\otimes\phi(z_i),\quad
\mathbf{g}_1=\dfrac{1}{n}\sum_{i=1}^n \phi(z_i)\otimes  \psi(x_i)
$$
\end{theorem}
We prove an original, finite sample bound on the RKHS-norm distance of the estimator $E^n_{\lambda}$ from its target $E_{\rho}$.
The proof is in Appendix~\ref{sec:stage1smaleconvergenceproofs}.
\begin{theorem}\label{stage1}
Assume Hypotheses 2-5. $\forall \delta\in(0,1)$, the following holds w.p. $1-\delta$:
\begin{align*}
    \|E^n_{\lambda}-E_{\rho}\|_{\mathcal{H}_{\Gamma}}&\leq r_E(\delta,n,c_1):= \dfrac{ \sqrt{\zeta_1}(c_1+1)}{4^{\frac{1}{c_1+1}}} \bigg(\dfrac{4\kappa(Q+\kappa \|E_{\rho}\|_{\mathcal{H}_{\Gamma}}) \ln(2/\delta)}{ \sqrt{n\zeta_1}(c_1-1)}\bigg)^{\frac{c_1-1}{c_1+1}}  \\
    \lambda&=\bigg(\dfrac{8\kappa(Q+\kappa \|E_{\rho}\|_{\mathcal{H}_{\Gamma}}) \ln(2/\delta)}{ \sqrt{n\zeta_1}(c_1-1)}\bigg)^{\frac{2}{c_1+1}}
\end{align*}
\end{theorem}
The efficient rate of $\lambda$ is $n^{\frac{-1}{c_1+1}}$. Note that the convergence rate of $E^n_{\lambda}$ is calibrated by $c_1$, which measures the smoothness of the conditional expectation operator $E_{\rho}$.

\subsection{Stage 2}



\textbf{Integral operators:} We use integral operator notation from the kernel methods literature, adapted to the structural operator learning problem. We denote by $L^2(\mathcal{H}_{\mathcal{X}},\rho_{\mathcal{H}_{\mathcal{X}}})$ the space of square integrable functions from $\mathcal{H}_{\mathcal{X}}$ to $\mathcal{Y}$ 
with respect to measure $\rho_{\mathcal{H}_{\mathcal{X}}}$, where $\rho_{\mathcal{H}_{\mathcal{X}}}$ is the extension of $\rho$ to $\mathcal{H}_{\mathcal{X}}$ \cite[Lemma A.3.16]{steinwart2008support}. Note that we present stage 2 analysis for general output space $\mathcal{Y}$ as in \cite{szabo2015two,szabo2016learning}, though in practice we only consider $\mathcal{Y}\subset \mathbb{R}$ to simplify our two-stage RKHS model.
\begin{definition}\label{def:stage2operators}
The stage 2 (population) operators are
\begin{align*}
 S^*&:\mathcal{H}_{\Omega}\hookrightarrow L^2(\mathcal{H}_{\mathcal{X}}, \rho_{\mathcal{H}_{\mathcal{X}}}),\enskip H\mapsto \Omega^*_{(\cdot)}H \\
    S&:L^2(\mathcal{H}_{\mathcal{X}}, \rho_{\mathcal{H}_{\mathcal{X}}})\rightarrow\mathcal{H}_{\Omega},\enskip \tilde{H}\mapsto \int \Omega_{\mu(z)} \circ\tilde{H}\mu(z) d\rho_{\mathcal{H}_{\mathcal{X}}}(\mu(z))
\end{align*}
\end{definition}
where $\Omega_{\mu(z)}: \mathcal{Y}\rightarrow\mathcal{H}_{\Omega}$ defined by $y\mapsto \Omega(\cdot,\mu(z))y$ is the point evaluator of \cite{micchelli2005learning,carmeli2006vector}. Finally define $T_{\mu(z)}=\Omega_{\mu(z)}\circ \Omega^*_{\mu(z)}$ and covariance operator $T=S\circ  S^*$.

\textbf{Assumptions:} We place assumptions on the original space $\mathcal{Y}$, the scalar-valued RKHS $\mathcal{H}_{\Omega}$, and the probability distribution $\rho$. Importantly, we assume that the scalar-valued
RKHS regression is correctly specified: the true structural operator $H_{\rho}$ lives in the scalar-valued RKHS $\mathcal{H}_{\Omega}$.


\begin{hypothesis}\label{hyp:polish}Suppose that $\mathcal{Y}$ is a Polish space
\end{hypothesis}

\begin{hypothesis}\label{hyp:hilbertSchmidt} Suppose that
\begin{enumerate}
   \item  The $\{\Omega_{\mu(z)}\}$ operator family is uniformly bounded in Hilbert-Schmidt norm: $\exists B$ s.t. $\forall \mu(z)$, 
    $
    \|\Omega_{\mu(z)}\|^2_{\mathcal{L}_2(\mathcal{Y},\mathcal{H}_{\Omega})}=Tr(\Omega_{\mu(z)}^*\circ \Omega_{\mu(z)})\leq B
    $
    \item The $\{\Omega_{\mu(z)}\}$ operator family is H\"older continuous in operator norm: $\exists L>0$, $\iota \in(0,1]$ s.t. $\forall \mu(z),\mu(z')$, 
    $
    \|\Omega_{\mu(z)}-\Omega_{\mu(z')}\|_{\mathcal{L}(\mathcal{Y},\mathcal{H}_{\Omega})} \leq L \|\mu(z)-\mu(z')\|^{\iota}_{\mathcal{H}_{\mathcal{X}}}
    $
\end{enumerate}
\end{hypothesis}
Larger $\iota$ is interpretable as smoother kernel $\Omega$.

\begin{hypothesis}\label{hyp:boundedY} Suppose that 
\begin{enumerate}
    \item $H_{\rho}\in \mathcal{H}_{\Omega}$. Then $
        \mathcal{E}(H_{\rho})=\inf_{H\in\mathcal{H}_{\Omega}}\mathcal{E}(H)
    $
    \item $Y$ is bounded, i.e. $\exists C<\infty$ s.t. $\|Y\|_{\mathcal{Y}}\leq C$ almost surely
\end{enumerate}
\end{hypothesis}

The convergence rate from stage 1 together with Hypotheses 6-8 are sufficient to bound the excess error of the regularized estimator $\hat{H}_{\xi}^{m}$ in terms of familiar objects in the kernel methods literature, namely the residual, reconstruction error, and effective dimension. We further assume $\rho$ belongs to a stage 2 prior to simplify these bounds. In particular, we assume $\rho$ belongs to a class of distributions parametrized by $(\zeta,b,c)$ as defined originally in \cite[Definition 1]{caponnetto2007optimal}, restated below.
\begin{hypothesis}\label{hyp:caponnettoSmoothness}
Fix $\zeta<\infty$. For given $b\in(1,\infty]$ and $c\in(1,2]$, define the prior $\mathcal{P}(\zeta,b,c)$ as the set of probability distributions $\rho$ on $\mathcal{H}_{\mathcal{X}}\times \mathcal{Y}$ such that 
\begin{enumerate}
    \item A range space assumption is satisfied: $\exists G\in\mathcal{H}_{\Omega}$ s.t.
    $
    H_{\rho} =T^{\frac{c-1}{2}}G$ and $
    \|G\|^2_{\mathcal{H}_{\Omega}}\leq \zeta
  $
    \item In the spectral decomposition $
    T=\sum_{k=1}^{\infty} \lambda_k  e_k \langle \cdot, e_k\rangle_{\mathcal{H}_{\Omega}}
    $, where $\{e_k\}_{k=1}^{\infty}$ is a basis of $Ker(T)^{\perp}$, the eigenvalues satisfy $\alpha\leq k^b \lambda_k \leq \beta$ for some $\alpha,\beta>0$
\end{enumerate}
\end{hypothesis}
We define the power of operator $T$ with respect to its eigendecomposition; see Appendix~\ref{sec:cov_technical} for formal justification. The latter condition is interpretable as polynomial decay of eigenvalues: $\lambda_k=\Theta(k^{-b})$. Larger $b$ means faster decay of eigenvalues of the covariance operator $T$ and hence smaller effective input dimension. Larger $c$ corresponds to a smoother structural operator $H_{\rho}$ \cite{szabo2016learning}.


\textbf{Estimation and convergence:} The estimator has a closed form solution, as shown by \cite{szabo2015two,szabo2016learning}  in the second stage of the distribution regression problem. We present the solution in notation similar to \cite{caponnetto2007optimal} to elucidate how the stage 1 and stage 2 estimators have the same structure.
\begin{theorem}\label{sol_2}
$\forall \xi>0$, the solution $H^{m}_{\xi}$ to $\mathcal{E}_{\xi}^m$ 
and the solution $\hat{H}^{m}_{\xi}$ to $\hat{\mathcal{E}}_{\xi}^m$ exist, are unique, and
\begin{align*}
    H_{\xi}^{m}&=(\mathbf{T}+\xi)^{-1}\mathbf{g},\quad \mathbf{T}=\dfrac{1}{m}\sum_{i=1}^m T_{\mu(\tilde{z}_i)},\quad \mathbf{g}=\dfrac{1}{m}\sum_{i=1}^m \Omega_{\mu(\tilde{z}_i)}\tilde{y}_i \\
    \hat{H}_{\xi}^{m}&=(\hat{\mathbf{T}}+\xi)^{-1}\hat{\mathbf{g}},\quad \hat{\mathbf{T}}=\dfrac{1}{m}\sum_{i=1}^m T_{\mu^n_{\lambda}(\tilde{z}_i)},\quad \hat{\mathbf{g}}=\dfrac{1}{m}\sum_{i=1}^m \Omega_{\mu^n_{\lambda}(\tilde{z}_i)}\tilde{y}_i
\end{align*}
\end{theorem}

We now present this paper's main theorem. In Appendix~\ref{sec:finalTheorems}, we provide a finite sample bound on the excess error of the estimator $\hat{H}^{m}_{\xi}$ with respect to its target $H_{\rho}$.
Adapting arguments by \cite{szabo2016learning}, we demonstrate that KIV is able to achieve the minimax optimal single-stage rate derived by \cite{caponnetto2007optimal}. In other words, our two-stage estimator is able to learn the causal relationship with confounded data equally well as single-stage RKHS regression is able to learn the causal relationship with unconfounded data.
\begin{theorem}\label{rate}
Assume Hypotheses 1-9. Choose $\lambda=n^{-\frac{1}{c_1+1}}$ and $n=m^{\frac{a(c_1+1)}{\iota(c_1-1)}}$ where $a>0$.
\begin{enumerate}
    \item If $a\leq \frac{b(c+1)}{bc+1}$ then $\mathcal{E}(\hat{H}_{\xi}^{m})-\mathcal{E}(H_{\rho})=O_p(m^{-\frac{ac}{c+1}})$ with $\xi=m^{-\frac{a}{c+1}}$
    \item If $a\geq \frac{b(c+1)}{bc+1}$ then $\mathcal{E}(\hat{H}_{\xi}^{m})-\mathcal{E}(H_{\rho})=O_p(m^{-\frac{bc}{bc+1}})$ with $\xi=m^{-\frac{b}{bc+1}}$
\end{enumerate}
\end{theorem}
At $a=\frac{b(c+1)}{bc+1}<2$, the convergence rate $m^{-\frac{bc}{bc+1}}$ is minimax optimal while requiring the fewest observations \cite{szabo2016learning}. This statistically efficient rate is calibrated by $b$, the effective input dimension, as well as $c$, the smoothness of structural operator $H_{\rho}$ \cite{caponnetto2007optimal}. The efficient ratio between stage 1 and stage 2 samples is $n=m^{\frac{b(c+1)}{bc+1}\cdot \frac{(c_1+1)}{\iota(c_1-1)}}$, implying $n>m$. As far as we know, asymmetric sample splitting is a novel prescription in the IV literature; previous analyses assume $n=m$ \cite{angrist1995split,hefny2015supervised}.

\section{Experiments}\label{sec:experiments}

We compare the empirical performance of KIV (\verb|KernelIV|) to four leading competitors: standard kernel ridge regression (\verb|KernelReg|) \cite{saunders1998ridge}, Nadaraya-Watson IV (\verb|SmoothIV|) \cite{carrasco2007linear,darolles2011nonparametric}, sieve IV (\verb|SieveIV|) \cite{newey2003instrumental,chen2018optimal}, and deep IV (\verb|DeepIV|) \cite{hartford2017deep}. To improve the performance of sieve IV, we impose Tikhonov regularization in both stages with KIV's tuning procedure. This adaptation exceeds the theoretical justification provided by \cite{chen2018optimal}. However, it is justified by our analysis insofar as sieve IV is a special case of KIV: set feature maps $\psi,\phi$ equal to the sieve bases.


    \begin{wrapfigure}{R}{\textwidth/3}
\vspace{5pt}
  \begin{center}
    \includegraphics[width=\textwidth/3]{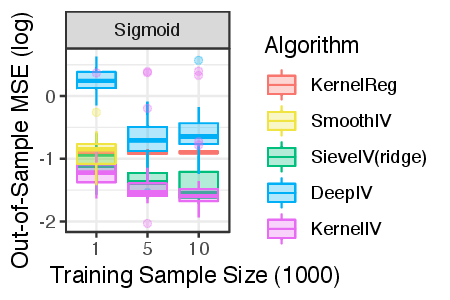}
     \vspace{-10pt}
    \caption{Sigmoid design}
    \label{sigmoid_main}
    \end{center}
        \vspace{-0pt}
\end{wrapfigure}

We implement each estimator on three designs. The \textit{linear} design \cite{chen2018optimal} involves learning counterfactual function $h(x)=4x-2$, given confounded observations of continuous variables $(X,Y)$ as well as continuous instrument $Z$. The \textit{sigmoid} design \cite{chen2018optimal} involves learning counterfactual function $h(x)=\ln (|16x-8|+1)\cdot sgn(x-0.5)$ under the same regime. The \textit{demand} design \cite{hartford2017deep} involves learning demand function $h(p,t,s)=100+(10+p)\cdot s\cdot \psi(t)-2p$ where $
\psi(t)$ is the complex nonlinear function in Figure~\ref{demand_dgp}. An observation consists of $(Y,P,T,S,C)$ where $Y$ is sales, $P$ is price, $T$ is time of year, $S$ is customer sentiment (a discrete variable), and $C$ is a supply cost shifter. The parameter $\rho\in\{0.9,0.75,0.5,0.25,0.1\}$ calibrates the extent to which price $P$ is confounded by supply-side market forces.  In KIV notation, inputs are $X=(P,T,S)$ and instruments are $Z=(C,T,S)$.

    
     For each algorithm, design, and sample size, we implement 40 simulations and calculate MSE with respect to the true structural function $h$. Figures~\ref{sigmoid_main},~\ref{demand}, and~\ref{uni} visualize results. 
    In the sigmoid design, \verb|KernelIV| performs best across sample sizes. In the demand design, \verb|SmoothIV| performs best for sample size $n+m=1000$. Like \cite{hartford2017deep}, we do not implement \verb|SmoothIV| for greater sample sizes due to running time. Among estimators that we are able to implement, \verb|KernelIV| performs best for sample sizes $n+m=5000$ and $n+m=10000$. \verb|KernelReg| ignores the instrument $Z$, and it is biased away from the structural function due to confounding noise $e$. This phenomenon can have counterintuitive consequences. Figure~\ref{demand} shows that in the highly nonlinear demand design, \verb|KernelReg| deviates further from the structural function as sample size increases because the algorithm is further misled by confounded data. Figure 2 of \cite{hartford2017deep} documents the same effect when a feedforward neural network is used. The remaining algorithms make use of the instrument $Z$ to overcome this issue.
    
    \verb|KernelIV| improves on \verb|SieveIV| in the same way that kernel ridge regression improves on ridge regression: by using an infinite dictionary of implicit basis functions rather than a finite dictionary of explicit basis functions. \verb|KernelIV| improves on \verb|SmoothIV| by using kernel ridge regression in not only stage 2 but also stage 1, avoiding costly density estimation. Finally, it improves on \verb|DeepIV| by directly learning stage 1 mean embeddings, rather than performing costly density estimation and sampling from the estimated density. The experiments show that \verb|KernelIV| works particularly well when the true structural function $h$ is smooth, confirming the theoretical guarantee of Theorem~\ref{rate}. See Appendix~\ref{sec:sim} for representative plots, implementation details, and a robustness study.

\begin{figure}[H]
\centering
\vspace{-5pt}
    \includegraphics[width=\textwidth]{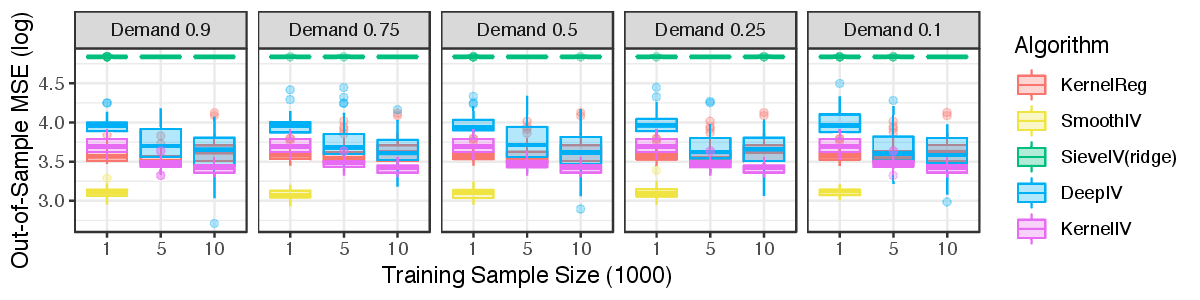}
    \vspace{-15pt}
    \caption{Demand design}
    \label{demand}
    \vspace{-5pt}
\end{figure}

\section{Conclusion}

We introduce KIV, an algorithm for learning a nonlinear, causal relationship from confounded observational data. KIV is easily implemented and minimax optimal. As a contribution to the IV literature, we show how to estimate the stage 1 conditional expectation operator--an infinite by infinite dimensional object--by kernel ridge regression. As a contribution to the kernel methods literature, we show how the RKHS is well-suited to causal inference and ill-posed inverse problems.
In simulations, KIV outperforms state of the art algorithms for nonparametric IV regression. The success of KIV suggests RKHS methods may be an effective bridge between econometrics and machine learning.

\subsubsection*{Acknowledgments}

We are grateful to Alberto Abadie, Anish Agarwal, Michael Arbel, Victor Chernozhukov, Geoffrey Gordon, Jason Hartford, Motonobu Kanagawa, Anna Mikusheva, Whitney Newey, Nakul Singh, Bharath Sriperumbudur, and Suhas Vijaykumar. This project was made possible by the Marshall Aid Commemoration Commission.

\newpage


\appendix

\section{Appendix}

\localtableofcontents

\newpage

\subsection{Instrumental variable}

\subsubsection{Comparison of IV assumptions}\label{sec:simpleIVadditiveNoise}

Here, we compare Hypothesis~\ref{iv} with two alternative formulations of the IV assumption. 

\begin{wrapfigure}{R}{0.33\textwidth}
\vspace{-5pt}
\begin{center}
\begin{adjustbox}{width=.33\textwidth}
\begin{tikzpicture}[->,>=stealth',shorten >=1pt,auto,node distance=2.8cm,
                    semithick]
  \tikzstyle{every state}=[draw=black,text=black]

  \node[state]         (z) [fill=gray]                   {$Z$};
  \node[state]         (x) [right of=z, fill=gray]       {$X$};
  \node[state]         (y) [right of=x, fill=gray]       {$Y$};
   \node[state]         (e) [above left of=y]                  {$e$};

  \path (z) edge              node {$ $} (x)
        (x) edge              node {$ $} (y)
        (e) edge           node {$ $} (x)
       edge           node {$ $} (y);;
\end{tikzpicture}
\end{adjustbox}
\vspace{-10pt}
\caption{IV DAG}
\label{dag}
\end{center}
\vspace{-10pt}
\end{wrapfigure}

We refer to the first formulation in the introduction: conditional independence. This formulation consists of the following assumptions: exclusion $Z\indep Y|(X,e)$; unconfounded instrument $Z\indep e$; and relevance, i.e. $\rho(x|z)$ is not constant in $z$. The directed acyclic graph (DAG) in Figure~\ref{dag} encodes these assumptions. Definition 7.4.1 of \cite{pearl2009causality} provides a formal graphical criterion. 

The second formulation is via potential outcomes \cite{angrist1996identification}. Though it is beyond the scope of this work, see \cite[Chapter 7]{hernan2019causal} for the relation between DAGs and potential outcomes.

We use a third formulation, which belongs in the moment restriction framework for causal inference. In the moment restriction approach, we encode causal assumptions via functional form restrictions and conditional expectations set to zero. Hypothesis~\ref{iv}, introduced by \cite{newey2003instrumental}, involves such statements. In particular, it imposes additive separability of confounding noise $e$, and $\mathbb{E}[e|Z]=0$. Be imposing the former, we can relax the independences $Z\indep Y|(X,e)$ and $Z\indep e$ to mean independence $\mathbb{E}[e|Z]=0$.

We recommend \cite{singh2019causal} for a comparison of the DAG, potential outcome, and moment restriction frameworks for causal inference.

\subsubsection{Linear vignette}

To build intuition for the IV model, we walk through a classic vignette about the linear case. We show how least squares (LS) has a different estimand than two-stage least squares (2SLS) when observations are confounded, i.e. with confounding noise. We will see that the estimand of 2SLS is the structural parameter of interest.


Consider the model
$$
Y=\beta'X+e,\quad \mathbb{E}[Xe]\neq0, \quad \mathbb{E}[e|Z]=0
$$
where $Y,e\in \mathbb{R}$, $X\in \mathbb{R}^{d_x}$, $Z\in \mathbb{R}^{d_z}$,  and  $d_z\geq d_x$. Data $(X,Y)$ are confounded but we have access to instrument $Z$. We aim to recover structural parameter $\beta$. Denote the estimands of LS and 2SLS by $\beta^{LS}$ and $\beta^{2SLS}$, respectively. For clarity, we write the variables to which expectations refer.

\begin{proposition}\label{vignette}
$\beta^{LS}\neq\beta=\beta^{2SLS}$
\end{proposition}

\begin{proof}
$\beta^{LS}$ is the projection of $Y$ onto $X$.
$$
\beta^{LS}=\mathbb{E}_X[XX']^{-1}\mathbb{E}_{X,Y}[XY]=\beta+\mathbb{E}_{X}[XX']^{-1}\mathbb{E}_{X,e}[Xe]\neq\beta
$$
where the second equality substitutes $Y=X'\beta+e$.

Define $\bar{X}(Z):=\mathbb{E}[X|Z]$ and $\bar{Y}(Z):=\mathbb{E}[Y|Z]$. $\beta^{2SLS}$ is the projection of $Y$ onto $\bar{X}(Z)$.
$$
\beta^{2SLS}=\mathbb{E}_Z[\bar{X}(Z)\bar{X}(Z)']^{-1}\mathbb{E}_{Z,Y}[\bar{X}(Z)Y]
$$
Finally we confirm that $\beta^{2SLS}=\beta$. Taking $\mathbb{E}[\cdot|Z]$ of the model LHS and RHS
$$
\bar{Y}(Z)=\bar{X}(Z)'\beta\implies \bar{X}(Z)\bar{Y}(Z)=\bar{X}(Z)\bar{X}(Z)'\beta \implies \mathbb{E}_{Z}[\bar{X}(Z)\bar{Y}(Z)]=\mathbb{E}_Z[\bar{X}(Z)\bar{X}(Z)']\beta
$$
Appealing to the definition of conditional expectation,
$$
\beta=\mathbb{E}_Z[\bar{X}(Z)\bar{X}(Z)']^{-1}\mathbb{E}_{Z}[\bar{X}(Z)\bar{Y}(Z)]= \mathbb{E}_Z[\bar{X}(Z)\bar{X}(Z)']^{-1}\mathbb{E}_{Z,Y}[\bar{X}(Z)Y]
$$
\end{proof}

The final equality in the proof makes an important point: in 2SLS, one may use projected outputs $\bar{Y}(Z)$ or original outputs $Y$ in stage 2. Choice of the latter simplifies estimation and analysis. 

In the present work, we extend this basic model and approach. We consider inputs $\psi(X)$ instead of $X$ and instruments $\phi(Z)$ instead of $Z$. Matching symbols, the model becomes
$$Y=h(X)+e=H\psi(X)+e$$
where the structural operator $H$ generalizes the structural parameter $\beta$. Whereas 2SLS regresses $Y$ on $\bar{X}(Z)=\mathbb{E}[X|Z]$, KIV regresses $Y$ on $\mu(Z)=\mathbb{E}[\psi(X)|Z]$.




\subsection{Comparison of nonparametric IV bounds}

In this section, we compare KIV with alternative nonparametric IV methods that have statistical
guarantees. Readers may find it helpful to familiarize themselves with our results in Section~\ref{sec:consistency} before reading this section.

\subsubsection{Nadaraya-Watson IV}\label{sec:comparisonDarollesNonparametric}

We first give a detailed account of the bound for nonparametric two-stage IV regression in \cite{darolles2011nonparametric}, which provides an
explicit end-to-end rate for the combined stages 1 and 2. In this work, stage 1 requires estimates of the conditional density
of the input $X$ and output $Y$ given the instrument $Z$. Stage
2 is a ridge regression performed in the relevant space of square
integrable functions; the ridge penalty is not directly on RKHS norm, unlike the present work. Still, \cite[Assumption A.2]{darolles2011nonparametric} requires that the structural function $h$ is an element of an RKHS defined from the right singular values of the conditional expectation operator $E$ in order to prove consistency. To facilitate comparison between \cite{darolles2011nonparametric} and the present work, we present the operator equation in both notations
$$
\mathbb{E}[Y|Z]=Eh,\quad r=T\varphi
$$

The stage 1 rate of \cite[Assumption 3]{darolles2011nonparametric}
directly follows from the convergence rate for the Nadaraya-Watson conditional density estimate, expressed as a ratio of unconditional estimates. Definition 4.1 of \cite{darolles2011nonparametric} describes the density estimation kernels, which
should not be confused with RKHS kernels. The rate depends on the smoothness of the density (specifically, the number of derivatives that exist),
the dimension of the random variables, and the smoothness of the density
estimation kernel used. The combined stage 1 and 2 result in \cite[Theorem 4.1, Corollary 4.2]{darolles2011nonparametric}
requires a further smoothness assumption on the stage 2 regression function $h$, as outlined in \cite[Proposition 3.2]{darolles2011nonparametric}. Our smoothness assumption in Hypothesis \ref{hyp:caponnettoSmoothness} plays an analogous role, though it takes a different form.

There are a number of significant differences between \cite{darolles2011nonparametric} and KIV. Consider stage 1 of the learning problem. Density estimation is a more general task than computing conditional mean embeddings $\mu(z)=\mathbb{E}_{X|Z=z}\psi(X)$, which are all that stage 2 regression requires. In particular, density estimation
rapidly becomes more difficult with increasing dimension \cite[Section 6.5]{Wasserman06AllOF},
whereas the difficulty of learning $\mu(z)$ depends solely on the smoothness
of the regression function to $\mathcal{H}_{\mathcal{X}}$; recall Hypothesis~\ref{hyp:smaleSmoothness}. Thus,
when the input $X$ and instrumental variable $Z$ are in moderate to high
dimensions, we expect conditional density estimation in stage 1 of \cite{darolles2011nonparametric} to suffer a drop
in performance unlike kernel ridge regression in stage 1 of KIV. (As an aside, the approach to conditional density estimation that involves
a ratio of Nadaraya-Watson estimates is suboptimal; better direct
estimates of conditional densities exist \cite{Sugiyama2010leastsquares,arbel2018kernel,Dutordoir18Gaussian}.)

Finally, there is no discussion of whether the overall rate obtained
in \cite{darolles2011nonparametric} is optimal under the smoothess
assumptions made. Relatedly, there is no discussion of what an efficient ratio of stage 1 to stage 2 training samples might be. By contrast, our stage 2 result has a minimax optimal guarantee accompanied by a recommended ratio of training sample sizes.

\subsubsection{Kernel PSR}\label{sec:comparisonHefnyNonparametric}

Next we describe a bound for two-stage IV regression
derived in the context of predictive state representations (PSRs) \cite{hefny2015supervised}.
PSRs are a means of performing filtering and smoothing for a time
series of observations $o_{1},\ldots,o_{t}$. In this setting, future
observations are summarized as a feature vector $\varphi_{t}:=\varphi(o_{t:t+k-1})$,
and past observations as a feature vector $h_{t}:=h(o_{1:t-1})$.
The predictive state is the expectation of future features given the history: $q_{t}:=\mathbb{E}[\varphi_{t}|h_{t}]$. Features can be RKHS feature maps \cite{boots2013hilbert}. In this case, the predictive state is a conditional mean embedding.

Given history $q_{t}$, the goal of filtering is to predict the extended
future state $p_{t}:=\mathbb{E}[\xi_{t}|h_{t}]$, where
$\xi_{t}:=\xi(o_{t:t+k})$ \cite[eq. 2]{hefny2015supervised}. The
relation with IV regression is apparent: both $q_{t}$
and $p_{t}$ are the result of stage 1 regression, and the mapping
between them is the result of stage 2 regression. Theorem 2 of \cite{hefny2015supervised} gives
a finite sample bound for the final stage 2 result, which
incorporates convergence results for stage 1 from \cite[Theorem 6]{song2009hilbert}.

There are several key differences between the \cite{hefny2015supervised} bound and the KIV bound. First, the \cite{hefny2015supervised} bound does not make full use of the structure of the conditional mean embedding regression problem \cite{grunewalder2012conditional}. Rather, \cite{hefny2015supervised} apply matrix concentration results from \cite{Hsu2012tail}
to the operators used in constructing the regression function. As a consequence, the stage 2 rate is slower than the minimax optimal rate proposed in \cite{caponnetto2007optimal}.

Another consequence is that the analysis in \cite{hefny2015supervised} requires strong assumptions about the smoothness of the input
to stage 2 regression. By contrast, our regression-specific 
analysis requires assumptions on the smoothness of the regression
function; see \cite[Theorem 2]{smale2007learning}
and \cite[Definition 1]{caponnetto2007optimal}. The proof of \cite{hefny2015supervised}
  additionally assumes that the stage 2 regression is a Hilbert-Schmidt operator, which amounts to a smoothness assumption, however this is insufficient for their bound.
  
  We now show that the input smoothness assumptions
from \cite{hefny2015supervised} make the bound inapplicable in our
case. Suppose we wish to make a counterfactual prediction $y_{\mathrm{test}}=H\gamma_{\mathrm{test}}$
for some $\gamma_{\mathrm{test}}\in\mathcal{H}_{\mathcal{X}}$. From
\cite[Theorem 2]{hefny2015supervised}, the required assumption is that $\exists f_{\mathrm{test}}:\mathcal{X}\rightarrow\mathcal{H}_{\mathcal{X}}$ such that
$$
\gamma_{\mathrm{test}}=\int\left(\int \psi(x')d\rho(x'|z)\right)\left(\int f_{\mathrm{test}}(x)d\rho(x|z)\right)d\rho(z)
$$
Our final goal of counterfactual prediction at a single point requires $\gamma_{\mathrm{test}}=\psi(x_{\mathrm{test}})$,
which will only hold in the trivial case when $\rho(x'|z)\rho(z)$ represents
a single point mass. In the PSR setting, the
assumption is not vacuous since $\gamma_{\mathrm{test}}$ will not
be the kernel at a single test point; see \cite[Lemma 3]{hefny2015supervised}.
An identical issue arises in the stage 1 bound of \cite[Proposition C.2]{hefny2015supervised},
since it uses a result from \cite[Theorem 6]{song2009hilbert} which
makes an analogous input smoothness assumption. In summary, neither bound
applies in our setting.

Finally, \cite[Theorem 2]{hefny2015supervised} does not explicitly
determine an efficient ratio of stage 1 and stage 2 training samples. Instead, analysis assumes an equal number of training samples in each stage. By contrast, we give an efficient ratio between training sample sizes required to obtain the minimax optimal rate in stage 2.

Despite the difference in setting, we believe our approach may
be used to improve the results in \cite{hefny2015supervised}.


\subsection{Vector-valued RKHS}\label{sec:vectorValuedRKHS}


We briefly review the theory of vector-valued RKHS as it relates to the IV regression problem. The primary reference is the
appendix of \cite{grunewalder2013smooth}.

\begin{proposition}[Lemma 4.33 of \cite{steinwart2008support}]\label{separable}
Under Hypotheses 2-3, $\mathcal{H}_{\mathcal{X}}$ and  $\mathcal{H}_{\mathcal{Z}}$ are separable.
\end{proposition}

\begin{proposition}[Theorem A.2 of \cite{grunewalder2013smooth}]\label{prop:vectorValuedKernelWithIdentity}
Let $I_{\mathcal{H}_{\mathcal{Z}}}:\mathcal{H}_{\mathcal{Z}}\rightarrow\mathcal{H}_{\mathcal{Z}}$ be the identity operator. $\Gamma(h,h')=\langle h,h'\rangle_{\mathcal{H}_{\mathcal{X}}}I_{\mathcal{H}_{\mathcal{Z}}}$ is a kernel of positive type.
\end{proposition}

\begin{proposition}[Proposition 2.3 of \cite{carmeli2006vector}]
Consider a kernel of positive type $\Gamma:\mathcal{H}_{\mathcal{X}}\times \mathcal{H}_{\mathcal{X}} \rightarrow \mathcal{L}(\mathcal{H}_{\mathcal{Z}})$, where $\mathcal{L}(\mathcal{H}_{\mathcal{Z}})$ is the space of bounded linear operators from $\mathcal{H}_{\mathcal{Z}}$ to $\mathcal{H}_{\mathcal{Z}}$. It corresponds to a unique RKHS $\mathcal{H}_{\Gamma}$ with reproducing kernel $\Gamma$.
\end{proposition}

\begin{proposition}[Theorem B.1 of \cite{grunewalder2013smooth}]
Each $E\in\mathcal{H}_{\Gamma}$ is a bounded linear operator $E:\mathcal{H}_{\mathcal{X}}\rightarrow\mathcal{H}_{\mathcal{Z}}$.
\end{proposition}

\begin{proposition}\label{inner}
$\mathcal{H}_{\Gamma}= \mathcal{L}_2(\mathcal{H}_{\mathcal{X}},\mathcal{H}_{\mathcal{Z}})$ and the inner products are equal.
\end{proposition}
\begin{proof}
\cite[Theorem 13]{berlinet2011reproducing} and \cite[eq. 12]{grunewalder2012conditional}.
\end{proof}


\begin{proposition}[Theorem B.2 of \cite{grunewalder2013smooth}.]
If $\exists E,G\in\mathcal{H}_{\Gamma}$ s.t. $\forall x\in\mathcal{X},\;E\psi(x)=G\psi(x)$ then $E=G$. Furthermore, if $\psi(x)$ is continuous in $x$ then it is sufficient that $E\psi(x)=G\psi(x)$ on a dense subset of $\mathcal{X}$.
\end{proposition}

\begin{proposition}[Theorem B.3 of \cite{grunewalder2013smooth}]\label{iso1}
$\forall E\in \mathcal{H}_{\Gamma}$, $\exists E^*\in\mathcal{H}_{\Gamma^*}$ where $\mathcal{H}_{\Gamma^*}$ is the vector-valued RKHS with reproducing kernel $\Gamma^*(l,l')=\langle l,l'\rangle_{\mathcal{H}_{\mathcal{Z}}}I_{\mathcal{H}_{\mathcal{X}}}$. $\forall h\in\mathcal{H}_{\mathcal{X}}$ and $\forall \ell \in\mathcal{H}_{\mathcal{Z}}$, 
$$\langle E h,\ell\rangle_{\mathcal{H}_{\mathcal{Z}}}=\langle h,E^*\ell\rangle_{\mathcal{H}_{\mathcal{X}}}$$
The operator $A\circ E=E^*$ is an isometric isomorphism from $\mathcal{H}_{\Gamma}$ to $\mathcal{H}_{\Gamma^*}$; $\mathcal{H}_{\Gamma}\cong \mathcal{H}_{\Gamma^*}$ and $\|E\|_{\mathcal{H}_{\Gamma}}=\|E^*\|_{\mathcal{H}_{\Gamma^*}}$.
\end{proposition}

\begin{proposition}[Theorem B.4 of \cite{grunewalder2013smooth}.]
The set of self-adjoint operators in $\mathcal{H}_{\Gamma}$ is a closed linear subspace.
\end{proposition}

\begin{proposition}[Lemma 15 of \cite{ciliberto2016consistent}]\label{iso2}
 $\mathcal{H}_{\Gamma^*}$ is isometrically isomorphic to $\mathcal{H}_{\Xi}$, the vector-valued RKHS with reproducing kernel $\Xi(z,z')=k_{\mathcal{Z}}(z,z')I_{\mathcal{H}_{\mathcal{X}}}$. $\forall \mu\in\mathcal{H}_{\Xi}$, $\exists! E^*\in \mathcal{H}_{\Gamma^*}$ s.t.
$$
\mu(z)=E^*\phi(z),\quad\forall z\in\mathcal{Z}
$$
\end{proposition}



\begin{proposition}
$\mathcal{H}_{\mathcal{X}}$ is isometrically isomorphic to $\mathcal{H}_{\Omega}$, the scalar-valued RKHS with reproducing kernel $\Omega$ defined s.t.
$$
\Omega(\psi(x),\psi(x'))=k_{\mathcal{X}}(x,x')
$$
\end{proposition}


\begin{proof}
\cite[eq. 7]{szabo2016learning} and Figure~\ref{rkhs_s}.
\end{proof}

\begin{proposition}\label{prop:existenceOfConditionalMeanElement}
Under Hypothesis~\ref{hyp:measurableBoundedFeatures},
$$
 \mathbb{E}_{X|Z=z}h(X)=[Eh](z)=\langle h,\mu(z)\rangle_{\mathcal{H}_{\mathcal{X}}}=H\mu(z)
$$
\end{proposition}

\begin{proof}
Hypothesis \ref{hyp:measurableBoundedFeatures} implies that the feature map is Bochner integrable \textup{\cite[Definition A.5.20]{steinwart2008support}} for the conditional distributions considered: $\forall z\in\mathcal{Z}$, $\mathbb{E}_{X|Z=z} \|  \psi(X)  \|<\infty$. 

The first equality holds by definition of the conditional expectation operator $E$. The second equality follows from Bochner integrability of the feature map, since it allows us to exchange the order of expectation and dot product.
\begin{align*}
    \mathbb{E}_{X|Z=z}h(X) &= \mathbb{E}_{X|Z=z}\langle h, \psi(X) \rangle_{\mathcal{H}_{\mathcal{X}}} \\
    &=\langle h, \mathbb{E}_{X|Z=z} \psi(X) \rangle_{\mathcal{H}_{\mathcal{X}}} \\
    &=\langle h, \mu(z) \rangle_{\mathcal{H}_{\mathcal{X}}}
\end{align*}
To see the third equality, note that Riesz representation theorem implies that the inner product with a given element $h\in \mathcal{H}_{\mathcal{X}}$ is uniquely represented by a bounded linear functional $H$ on $\mathcal{H}_{\mathcal{X}}$.
\end{proof}

\begin{proposition}\label{well_spec_condition}
Our RKHS construction implies that
$$[E_{\rho} h](\cdot) = \mathbb{E}_{X|Z=(\cdot)} [h(X)] \in\mathcal{H}_{\mathcal{Z}},\quad \forall h\in\mathcal{H}_{\mathcal{X}}$$
\end{proposition}

\begin{proof}
After defining $\mathcal{H}_{\mathcal{X}}$ and $\mathcal{H}_{\mathcal{Z}}$, we define the conditional expectation operator $E:\mathcal{H}_{\mathcal{X}}\rightarrow \mathcal{H}_{\mathcal{Z}}$ such that
$[Eh](z)=\mathbb{E}_{X|Z=z}h(X)$. By construction, $\mathbb{E}_{X|Z=(\cdot)}f(X)\in \mathcal{H}_{\mathcal{Z}},\:\forall f\in\mathcal{H}_{\mathcal{X}}$. This is precisely the condition required for the surrogate risk $\mathcal{E}_1$ to coincide with the natural risk for the conditional expectation operator \cite{grunewalder2012conditional,grunewalder2013smooth}. As such, $E_{\rho}=\argmin\mathcal{E}_1(E)$ is the true conditional expectation operator. 
\end{proof}

\subsection{Covariance operator}

\subsubsection{Definitions}

\begin{definition}\label{def:mu_minus}
$\mu^-:\mathcal{Z}\rightarrow\mathcal{H}_{\mathcal{X}}$ is the function that satisfies
$$
\mu^-(z)=\mathbb{E}_{X|Z=z}\psi(X),\quad\forall z\in D_{\rho|\mathcal{Z}}
$$
where $D_{\rho|\mathcal{Z}}\subset\mathcal{Z}$ is the support of $Z$, and $\mu^-(z)=0$ otherwise. 
\end{definition}

\begin{proposition}[Lemma 8 of \cite{ciliberto2016consistent}]
Assume Hypotheses 2-3. $\mu^-\in L^2(\mathcal{Z},\mathcal{H}_{\mathcal{X}},\rho_{\mathcal{Z}})$, where $L^2(\mathcal{Z},\mathcal{H}_{\mathcal{X}},\rho_{\mathcal{Z}})$ is the space of square integrable functions from $\mathcal{Z}$ to $\mathcal{H}_{\mathcal{X}}$ with respect to measure $\rho_{\mathcal{Z}}$.
\end{proposition}

\begin{definition}
Additional stage 1 population operators are
\begin{align*}
\tilde{T}_1&=S_1^*\circ S_1\\
R_1^*&:\mathcal{H}_{\mathcal{X}}\rightarrow L^2(\mathcal{Z},\rho_{\mathcal{Z}}) \\
&:h\mapsto \langle h,\mu^-(\cdot) \rangle_{\mathcal{H}_{\mathcal{X}}} \\
R_1&:L^2(\mathcal{Z},\rho_{\mathcal{Z}})\rightarrow\mathcal{H}_{\mathcal{X}} \\
&:\tilde{\ell}\mapsto \int \mu^-(z) \tilde{\ell}(z)d\rho_{\mathcal{Z}}(z) \\
T_{ZX}&=S_1\circ R_1^*
\end{align*}
\end{definition}
$T_{ZX}$ is the uncentered cross-covariance operator of \cite[Theorem 1]{fukumizu2004dimensionality}. The formulation as $S_1\circ R_1^*$ relates this integral operator to the integral operators in \cite{ciliberto2016consistent}.

\begin{definition}
The stage 1 empirical operators are
\begin{align*}
\hat{S}_1^*&:\mathcal{H}_{\mathcal{Z}}\rightarrow \mathbb{R}^n \\
&:\ell\mapsto \dfrac{1}{\sqrt{n}} \{\langle \ell,\phi(z_i) \rangle_{\mathcal{H}_{\mathcal{Z}}}\}_{i=1}^n \\
\hat{S}_1&:\mathbb{R}^n\rightarrow\mathcal{H}_{\mathcal{Z}} \\
&:\{v_i\}_{i=1}^n\mapsto  \dfrac{1}{\sqrt{n}}\sum_{i=1}^n\phi(z_i) v_i\\
\mathbf{T}_1&=\hat{S}_1\circ \hat{S}_1^* \\
\tilde{\mathbf{T}}_1&=\hat{S}_1^*\circ \hat{S}_1 \\
\hat{R}_1^*&:\mathcal{H}_{\mathcal{X}}\rightarrow \mathbb{R}^n \\
&:h\mapsto \dfrac{1}{\sqrt{n}} \{\langle h,\psi(x_i) \rangle_{\mathcal{H}_{\mathcal{X}}}\}_{i=1}^n \\
\hat{R}_1&:\mathbb{R}^n\rightarrow\mathcal{H}_{\mathcal{X}} \\
&:\{v_i\}_{i=1}^n\mapsto  \dfrac{1}{\sqrt{n}}\sum_{i=1}^n\psi(x_i)v_i \\
\mathbf{T}_{ZX}&=\hat{S}_1\circ \hat{R}_1^*
\end{align*}
\end{definition}
$\hat{S}_1^*$ is the sampling operator of \cite{smale2007learning}. $\mathbf{T}_1$ is the scatter matrix, while $K_{ZZ}=n\tilde{\mathbf{T}}_1$ is the empirical kernel matrix with respect to $Z$ as in \cite{ciliberto2016consistent}. Note that $\mathbf{T}_{ZX}=\mathbf{g}_1$ in Theorem~\ref{sol_1}.

\subsubsection{Existence and eigendecomposition}\label{sec:cov_technical}

We initially abstract from the problem at hand to state useful lemmas. Recall tensor product notation: if $a,b\in \mathcal{H}_1$ and $c\in\mathcal{H}_2$ then $[c\otimes a]b=c\langle a,b\rangle_{\mathcal{H}_1}$. Denote by $\mathcal{L}_2(\mathcal{H}_1,\mathcal{H}_2)$ the space of Hilbert-Schmidt operators from $\mathcal{H}_1$ to $\mathcal{H}_2$.

\begin{proposition}[eq. 3.6 of \cite{gretton2018notes}]\label{T1_lemma1}
If $\mathcal{H}_1$ and $\mathcal{H}_2$ are separable RKHSs, then 
$$\|c\otimes a\|_{\mathcal{L}_2(\mathcal{H}_1,\mathcal{H}_2)}=\|a\|_{\mathcal{H}_1}\|c\|_{\mathcal{H}_2}$$ 
and $c\otimes a\in\mathcal{L}_2(\mathcal{H}_1,\mathcal{H}_2)$.
\end{proposition}

\begin{proposition}[eq. 3.7 of \cite{gretton2018notes}]\label{T1_lemma2}
Assume $\mathcal{H}_1$ and $\mathcal{H}_2$ are separable RKHSs. If $C\in  \mathcal{L}_2(\mathcal{H}_1,\mathcal{H}_2)$ then 
$$\langle C,c\otimes a \rangle_{\mathcal{L}_2(\mathcal{H}_1,\mathcal{H}_2)}=\langle c, C a \rangle_{\mathcal{H}_2}$$
\end{proposition}

In Hypothesis 2, we assume that input space $\mathcal{X}$ and instrument space $\mathcal{Z}$ are separable. In Hypothesis 3, we assume RKHSs $\mathcal{H}_{\mathcal{X}}$ and $\mathcal{H}_{\mathcal{Z}}$ have continuous, bounded kernels $k_{\mathcal{X}}$ and $k_{\mathcal{Z}}$ with feature maps $\psi$ and $\phi$, respectively. By Proposition~\ref{separable}, it follows that $\mathcal{H}_{\mathcal{X}}$ and $\mathcal{H}_{\mathcal{Z}}$ are separable, i.e. they have countable orthonormal bases that we now denote $\{e^{\mathcal{X}}_i\}_{i=1}^{\infty}$ and $\{e^{\mathcal{Z}}_i\}_{i=1}^{\infty}$. 

Denote by $\mathcal{L}_2(\mathcal{H}_{\mathcal{X}},\mathcal{H}_{\mathcal{Z}})$ the space of Hilbert-Schmidt operators $E:\mathcal{H}_{\mathcal{X}}\rightarrow \mathcal{H}_{\mathcal{Z}}$ with inner product $\langle E,G\rangle_{\mathcal{L}_2(\mathcal{H}_{\mathcal{X}},\mathcal{H}_{\mathcal{Z}})}=\sum_{i=1}^{\infty} \langle Ee^{\mathcal{X}}_i,Ge^{\mathcal{X}}_i \rangle_{\mathcal{H}_{\mathcal{Z}}}$. Denote by $\mathcal{L}_2(\mathcal{H}_{\mathcal{Z}},\mathcal{H}_{\mathcal{Z}})$ the space of Hilbert-Schmidt operators $A:\mathcal{H}_{\mathcal{Z}}\rightarrow \mathcal{H}_{\mathcal{Z}}$ with inner product $\langle A,B\rangle_{\mathcal{L}_2(\mathcal{H}_{\mathcal{Z}},\mathcal{H}_{\mathcal{Z}})}=\sum_{i=1}^{\infty} \langle Ae^{\mathcal{Z}}_i,Be^{\mathcal{Z}}_i \rangle_{\mathcal{H}_{\mathcal{Z}}}$. When it is contextually clear, we abbreviate both spaces as $\mathcal{L}_2$.

\begin{proposition}\label{T1_exists}
Assume Hypotheses 2-3. $\exists T_{ZX}\in \mathcal{L}_2(\mathcal{H}_{\mathcal{X}},\mathcal{H}_{\mathcal{Z}})$ and $\exists T_1\in \mathcal{L}_2(\mathcal{H}_{\mathcal{Z}},\mathcal{H}_{\mathcal{Z}})$ s.t.
\begin{align*}
\langle T_{ZX},E\rangle_{\mathcal{L}_2}&=\mathbb{E}\langle \phi(Z)\otimes \psi(X),E\rangle_{\mathcal{L}_2} \\
   \langle T_1,A\rangle_{\mathcal{L}_2}&=\mathbb{E}\langle \phi(Z)\otimes \phi(Z),A\rangle_{\mathcal{L}_2} 
\end{align*}

\end{proposition}

\begin{proof}
By Riesz representation theorem, $T_{ZX}$ and $T_1$ exist if the RHSs are bounded linear operators. Linearity follows by definition. Boundedness follows since
\begin{align*}
    |\mathbb{E}\langle \phi(Z)\otimes \psi(X),E\rangle_{\mathcal{L}_2}|&\leq \mathbb{E}|\langle \phi(Z)\otimes \psi(X),E\rangle_{\mathcal{L}_2}|\leq \|E\|_{\mathcal{L}_2} \mathbb{E}\|\phi(Z)\otimes \psi(X)\|_{\mathcal{L}_2}\leq  \kappa Q  \|E\|_{\mathcal{L}_2} \\
    |\mathbb{E}\langle \phi(Z)\otimes \phi(Z),A\rangle_{\mathcal{L}_2}|&\leq \mathbb{E}|\langle \phi(Z)\otimes \phi(Z),A\rangle_{\mathcal{L}_2}|\leq \|A\|_{\mathcal{L}_2} \mathbb{E}\|\phi(Z)\otimes \phi(Z)\|_{\mathcal{L}_2}\leq  \kappa^2  \|A\|_{\mathcal{L}_2} 
\end{align*}
by Jensen, Cauchy-Schwarz, Proposition~\ref{T1_lemma1}, and boundedness of the kernels.
\end{proof}

\begin{proposition}\label{T1_cov}
Assume Hypotheses 2-3.
\begin{align*}
  \langle  \ell,T_{ZX}h\rangle_{\mathcal{H}_{\mathcal{Z}}}&=\mathbb{E}[\ell(Z)h(X)],\quad \forall \ell \in\mathcal{H}_{\mathcal{Z}},h\in\mathcal{H}_{\mathcal{X}} \\
    \langle  \ell,T_1\ell'\rangle_{\mathcal{H}_{\mathcal{Z}}}&=\mathbb{E}[\ell(Z)\ell'(Z)],\quad \forall \ell,\ell'\in\mathcal{H}_{\mathcal{Z}}
\end{align*}
\end{proposition}

\begin{proof}
\begin{align*}
    \langle  \ell,T_{ZX}h\rangle_{\mathcal{H}_{\mathcal{Z}}}
&=\langle T_{ZX},\ell\otimes h \rangle_{\mathcal{L}_2}
=\mathbb{E}\langle \phi(Z)\otimes \psi(X),\ell\otimes h \rangle_{\mathcal{L}_2}
=\mathbb{E}\langle \ell,\phi(Z)\rangle_{\mathcal{H}_{\mathcal{Z}}}\langle h,\psi(X)\rangle_{\mathcal{H}_{\mathcal{X}}}
 \\
    \langle  \ell,T_1\ell'\rangle_{\mathcal{H}_{\mathcal{Z}}}
&=\langle T_1,\ell\otimes \ell' \rangle_{\mathcal{L}_2}
=\mathbb{E}\langle \phi(Z)\otimes \phi(Z),\ell\otimes \ell' \rangle_{\mathcal{L}_2}
=\mathbb{E}\langle \ell,\phi(Z)\rangle_{\mathcal{H}_{\mathcal{Z}}}\langle \ell',\phi(Z)\rangle_{\mathcal{H}_{\mathcal{Z}}}
\end{align*}
by Proposition~\ref{T1_lemma2}, Proposition~\ref{T1_exists}, and Proposition~\ref{T1_lemma2}, respectively.
\end{proof}

\begin{proposition}\label{T1_tr}
Assume Hypotheses 2-3. 
\begin{align*}
    tr(T_{ZX})&\leq \kappa Q \\
    tr(T_1)&\leq \kappa^2
\end{align*}
\end{proposition}

\begin{proof}
\begin{align*}
    tr(T_{ZX})
&=\sum_{i=1}^{\infty} \langle e_i^{\mathcal{Z}},T_{ZX} e_i^{\mathcal{X}}\rangle_{\mathcal{H}_{\mathcal{Z}}}\\
&=\sum_{i=1}^\infty \mathbb{E}\langle e_i^{\mathcal{Z}},\phi(Z)\rangle_{\mathcal{H}_{\mathcal{Z}}}\langle e_i^{\mathcal{X}},\psi(X)\rangle_{\mathcal{H}_{\mathcal{X}}}\\
&=\mathbb{E}\sum_{i=1}^{\infty}\langle e_i^{\mathcal{Z}},\phi(Z)\rangle_{\mathcal{H}_{\mathcal{Z}}}\langle e_i^{\mathcal{X}},\psi(X)\rangle_{\mathcal{H}_{\mathcal{X}}}\\
&=\mathbb{E}\|\phi(Z)\|_{\mathcal{H}_{\mathcal{Z}}} \|\psi(X)\|_{\mathcal{H}_{\mathcal{X}}}\\
&\leq\kappa Q
\end{align*}
\begin{align*}
    tr(T_1)
&=\sum_{i=1}^{\infty} \langle e_i^{\mathcal{Z}},T_1 e_i^{\mathcal{Z}}\rangle_{\mathcal{H}_{\mathcal{Z}}}\\
&=\sum_{i=1}^\infty \mathbb{E}\langle e_i^{\mathcal{Z}},\phi(Z)\rangle^2_{\mathcal{H}_{\mathcal{Z}}}\\
&=\mathbb{E}\sum_{i=1}^{\infty} \langle e_i^{\mathcal{Z}},\phi(Z) \rangle^2_{\mathcal{H}_{\mathcal{Z}}}\\
&=\mathbb{E}\|\phi(Z)\|_{\mathcal{H}_{\mathcal{Z}}}^2\\
&\leq\kappa^2
\end{align*}
by definition of trace, the proof of Proposition~\ref{T1_cov}, monotone convergence theorem \cite[Theorem A.3.5]{steinwart2008support} with upper bounds $\kappa Q$ and $\kappa^2$, Parseval's identity, and boundedness of the kernels.
\end{proof}

Since stage 1 covariance operator $T_1$ has finite trace, its eigendecomposition is well-defined. Recall that the stage 2 covariance operator $T$ consists of functions from $\mathcal{H_{\mathcal{X}}}$ to $\mathcal{Y}=\mathbb{R}$. Since these functions have finite-dimensional output, it is immediate that $T$ has finite trace and its eigendecomposition is well-defined \cite[Remark 1]{caponnetto2007optimal}.

\begin{definition}\label{def_power}
The powers of operators $T_1$ and $T$ are defined as
\begin{align*}
    T_1^a 
     &=\sum_{k=1}^{\infty} \nu_k^a e^{\mathcal{Z}}_k\langle\cdot, e^{\mathcal{Z}}_k\rangle_{\mathcal{H}_{\mathcal{Z}}} \\
    T^a&=\sum_{k=1}^{\infty} \lambda^a_k  e_k \langle \cdot, e_k\rangle_{\mathcal{H}_{\Omega}}
\end{align*}
where $(\{\nu_k\},\{e^{\mathcal{Z}}_k\})$ is the spectrum of $T_1$ and $(\{\lambda_k\},\{e_k\})$ is the spectrum of $T$.
\end{definition}

\subsubsection{Properties}

\begin{proposition}\label{op1}
In this operator notation,
\begin{align*}
T_1&=\int_{\mathcal{Z}}\phi(z)\otimes\phi(z)d\rho_{\mathcal{Z}}(z) \\
T_{ZX}^*&=\int_{\mathcal{X}\times\mathcal{Z}}\psi(x)\otimes\phi(z)d\rho(x,z)  \\
T_{ZX}&=\int_{\mathcal{X}\times\mathcal{Z}}\phi(z)\otimes\psi(x)d\rho(x,z)
\end{align*}
\end{proposition}

\begin{proof}
\cite[Appendix A.1]{fukumizu2004dimensionality} or \cite[Proposition 13]{ciliberto2016consistent}. Note that
$$
\phi(z)\langle \psi(x),\cdot \rangle_{\mathcal{H}_{\mathcal{X}}}=[\phi(z)\otimes \psi(x)](\cdot)
$$
\end{proof}

\begin{proposition}\label{op2} 
Under Hypotheses 2-3
$$
T_{ZX}=T_1 \circ E_{\rho}
$$
\end{proposition}

\begin{proof}
\cite[Theorem 2]{fukumizu2004dimensionality}, appealing to Proposition~\ref{well_spec_condition}.
\end{proof}

Finally we state a property that will be useful for compositions involving covariance operators, generalizing \cite[Theorem 15]{bell2016handout}.

\begin{proposition}\label{opCS}
If $G\in\mathcal{L}_2(\mathcal{H}_{\mathcal{X}},\mathcal{H}_{\mathcal{Z}})$ and $B\in\mathcal{L}(\mathcal{H}_{\mathcal{Z}},\mathcal{H}_{\mathcal{Z}})$ then
$$
\|B\circ G\|_{\mathcal{L}_2}\leq \|B\|_{\mathcal{L}}\|G\|_{\mathcal{L}_2}
$$

\begin{proof}
$$
\|B\circ G\|^2_{\mathcal{L}_2}
=\sum_{i=1}^{\infty}\|B\circ G e_{i}^{\mathcal{X}}\|^2_{\mathcal{H}_{\mathcal{Z}}}
\leq \sum_{i=1}^{\infty}\left(\|B\|_{\mathcal{L} }\|Ge_i^{\mathcal{X}}\|_{\mathcal{H}_{\mathcal{Z}}}\right)^2
= \|B\|_{\mathcal{L}}^2\|G\|_{\mathcal{L}_2}^2
$$
where $\mathcal{L}$ is the operator norm and $\mathcal{L}_2$ is the Hilbert-Schmidt norm, and
the proof makes use of the operator norm definition.
\end{proof}

\end{proposition}


\subsubsection{Related work}\label{sec:cov_review}
Our approach allows both $\mathcal{H}_{\mathcal{X}}$ and $\mathcal{H}_{\mathcal{Z}}$ to be infinite-dimensional spaces.
Prior work on conditional mean embeddings and RKHS regression
has considered both  finite \cite{grunewalder2012conditional,grunewalder2013smooth}
and infinite \cite{song2009hilbert,song2010nonparametric,fukumizu2013kernel,hefny2015supervised,ciliberto2016consistent}
 dimensional RKHS $\mathcal{H}_{\mathcal{X}}$.
In this section, we briefly review this literature (besides the PSR case,
which we covered in Section \ref{sec:comparisonHefnyNonparametric}).


First, we recall results from Appendix~\ref{sec:vectorValuedRKHS}. $\mathcal{H}_{\Gamma}$ is a vector-valued RKHS consisting of operators $E:\mathcal{H}_{\mathcal{X}}\rightarrow\mathcal{H}_{\mathcal{Z}}$ with kernel $\Gamma(h,h')=\langle h, h'\rangle_{\mathcal{H}_{\mathcal{X}}}I_{\mathcal{H}_{\mathcal{Z}}}$. $\mathcal{H}_{\Xi}$ is a vector-valued RKHS consisting of mappings $\mu:\mathcal{Z}\rightarrow\mathcal{H}_{\mathcal{X}}$ with kernel $\Xi(z,z')=k_{\mathcal{Z}}(z,z')I_{\mathcal{H}_{\mathcal{X}}}$. By Propositions~\ref{iso1} and~\ref{iso2}, $\mathcal{H}_{\Gamma}$ and $\mathcal{H}_{\Xi}$ are isometrically isomorphic. There is a fundamental equivalence between $E$ and $\mu$, illustrated in Figure~\ref{rkhs_s}: $\mu(z)=E^*\phi(z)$.

Next, we present additional notation for vector-valued RKHS $\mathcal{H}_{\Xi}$.
\begin{align*}
    \Xi_z&:\mathcal{H}_{\mathcal{X}}\rightarrow \mathcal{H}_{\Xi} \\
    &:h\mapsto \Xi(\cdot,z)h=k_{\mathcal{Z}}(\cdot,z)h
\end{align*}
$\Xi_z$ is the point evaluator of \cite{micchelli2005learning,carmeli2006vector}. From this definition,
\begin{align*}
    \Xi(z,z')&=\Xi_z^*\circ \Xi_{z'} \\
    T_z^{\Xi}&=\Xi_z\circ \Xi_z^* \\
    T_1^{\Xi}&=\mathbb{E}T_z^{\Xi}
\end{align*}
and so $T_1^{\Xi}:\mathcal{H}_{\Xi}\rightarrow \mathcal{H}_{\Xi}$.

With this notation, we can communicate the constructions and assumptions of \cite{caponnetto2007optimal,grunewalder2012conditional}. In \cite[Hypothesis 1]{caponnetto2007optimal}, the authors assume $\Xi_z$ is a Hilbert-Schmidt operator. Definition 1 of \cite{caponnetto2007optimal} goes on to define the prior with respect to operator $T_1^{\Xi}$. The analysis of \cite{grunewalder2012conditional} inherits this framework. Section 6 of \cite{grunewalder2012conditional} further points out that $\Xi_z$ is not Hilbert-Schmidt if $\mathcal{H}_{\mathcal{X}}$ is infinite-dimensional since
$$
\|\Xi_z\|_{\mathcal{L}_2}=k_{\mathcal{Z}}(z,z)\sum_{i=1}^{\infty} \langle e_i^{\mathcal{X}},I_{\mathcal{H}_{\mathcal{X}}} e_i^{\mathcal{X}}\rangle_{\mathcal{H}_{\mathcal{X}}}=\infty
$$
Therefore the `main assumption' \cite[Table 1]{grunewalder2012conditional} is that $\mathcal{H}_{\mathcal{X}}$ is finite dimensional. The authors write, `It is likely that this assumption can be weakened, but this requires a deeper analysis'.

In the present work, we differ in our constructions and assumptions at this juncture. We instead focus on the covariance operator $T_1:\mathcal{H}_{\mathcal{Z}}\rightarrow\mathcal{H}_{\mathcal{Z}}$ as defined in \cite[Theorem 1]{fukumizu2004dimensionality}, previously applied to regression with an infinite-dimensional output space in
 \cite{song2009hilbert,song2010nonparametric,fukumizu2013kernel,hefny2015supervised,ciliberto2016consistent}.
Proposition~\ref{T1_tr} shows $tr(T_1)\leq \kappa^2$ under the mild assumptions in Hypotheses 2-3, so its eigendecomposition is well-defined. We place a prior with respect to $T_1$, and provide analysis inspired by \cite{smale2005shannon,smale2007learning} rather than \cite{caponnetto2007optimal}.

Specifically, in Hypothesis 4 we require that the stage 1 problem is well-specified: $E_{\rho}\in\mathcal{H}_{\Gamma}$. This requirement is stronger than the property articulated in Proposition~\ref{well_spec_condition}. Moreover, in Hypothesis 5 we assume 
$$
 E_{\rho}=T_1^{\frac{c_1-1}{2}}\circ G_1
$$
where $G_1:\mathcal{H}_{\mathcal{X}}\rightarrow \mathcal{H}_{\mathcal{Z}}$, $T_1^{\frac{c_1-1}{2}}:\mathcal{H}_{\mathcal{Z}}\rightarrow \mathcal{H}_{\mathcal{Z}}$, and $E_{\rho}:\mathcal{H}_{\mathcal{X}}\rightarrow \mathcal{H}_{\mathcal{Z}}$. By recognizing the equivalence of $E$ and $\mu$, we provide a general theory of conditional mean embedding regression in which $\mathcal{H}_{\mathcal{X}}$ is infinite. A question for further research is how to relax Hypothesis 4.

A number of previous works have studied consistency of the conditional expectation operator $E$ in the infinite-dimensional setting. Theorem 1 of \cite{song2010nonparametric} establishes consistency in Hilbert-Schmidt norm. However, the proof requires a strong smoothness assumption: that $T_1^{-3/2}\circ T_{ZX}$ is Hilbert-Schmidt. Theorem 8 of \cite{fukumizu2013kernel} establishes consistency of $E^*$ applied to embeddings of particular prior distributions, as needed to calculate a posterior by kernel Bayes' rule. The consistency results of \cite[Theorem 4, Theorem 5]{ciliberto2016consistent} for structured prediction are more relevant to our setting, and we discuss them in Appendix~\ref{sec:related_mu_bound} after establishing additional notation.

Finally, we remark that previous work has considered infinite-dimensional feature space in a broad variety of settings, beyond conditional mean embedding. In the setting of conditional density estimation, \cite{arbel2018kernel} propose an infinite-dimensional natural parameter for a conditional exponential family model, with a loss function derived from the Fisher score. See \cite[Lemma 1]{arbel2018kernel} for analysis specific to this particular loss.

\subsection{Algorithm}

\subsubsection{Derivation}\label{sec:alg_deriv}

\begin{proof}[Proof of Algorithm~\ref{alg}]
Rewrite the stage 1 regularized empirical objective as
\begin{align*}
E^n_{\lambda}&=\argmin_{E\in \mathcal{H}_{\Gamma}} \mathcal{E}_{\lambda}^n(E) \\
    \mathcal{E}_{\lambda}^n(E)
    &=\dfrac{1}{n}\sum_{i=1}^n \|\psi(x_i)-E^*\phi(z_i)\|^2_{\mathcal{H}_{\mathcal{X}}}+\lambda\|E\|^2_{\mathcal{L}_2(\mathcal{H}_{\mathcal{X}},\mathcal{H}_{\mathcal{Z}})} \\
    &=\dfrac{1}{n} \|\Psi_X-E^*\Phi_Z\|_{2}^2+\lambda\|E\|^2_{\mathcal{L}_2(\mathcal{H}_{\mathcal{X}},\mathcal{H}_{\mathcal{Z}})}
\end{align*}
where the $i^{th}$ column of $\Psi_X$ is $\psi(x_i)$ and the $i^{th}$ column of $\Phi_Z$ is $\phi(z_i)$. Hence by the standard regression formula
\begin{align*}
    (E^n_{\lambda})^*&=\Psi_X(K_{ZZ}+n\lambda I)^{-1}\Phi_Z' \\
    \mu^n_{\lambda}(z)&=(E^n_{\lambda})^*\phi(z)  \\
    &=\Psi_X(K_{ZZ}+n\lambda I)^{-1}\Phi_Z' \phi(z) \\
    &=\Psi_X\gamma(z) \\
    &=\sum_{i=1}^n\gamma_i(z)\psi(x_i)
\end{align*}
where
$$
\gamma(z):=(K_{ZZ}+n\lambda I)^{-1}\Phi_Z' \phi(z)=(K_{ZZ}+n\lambda I)^{-1}K_{Zz}
$$
Note that this expression coincides with the expression in Theorem~\ref{sol_1} after appealing to the proof of \cite[Proposition 2.1]{cortes2005general}.

By the representer theorem, we know that the first stage estimator $\mu_{\lambda}^n\in span(\{\psi(x_i)\})$ because we are effectively regressing $\{\phi(z_i)\}$ on $\{\psi(x_i)\}$ to learn the conditional expectation operator \cite{wahba1990spline,scholkopf2001generalized}. Indeed we have already shown
$$
\mu_{\lambda}^n (\cdot)=\sum_{j=1}^n\gamma_j(\cdot)\psi(x_j)
$$
In the second stage, we are effectively regressing on $\{\tilde{y}_i\}$ on $\mu_{\lambda}^n(\tilde{z}_i)$ to learn the structural function. By the representer theorem, then, $\hat{h}_{\xi}^m \in span(\{\mu_{\lambda}^n(\tilde{z}_i)\})$. But $\mu_{\lambda}^n(\tilde{z}_i)\in span(\{\psi(x_i)\})$, so $\hat{h}_{\xi}^m \in span(\{\psi(x_i)\})$. Thus the solution will take the form
$$
\hat{h}_{\xi}^m(\cdot)=\sum_{i=1}^n\alpha_i \psi(x_i)
$$
Substituting in this functional form as well as the solution for $\mu^n_{\lambda}$ permits us to rewrite
\begin{align*}
    [E^n_{\lambda}\hat{h}_{\xi}^m](z)&=\langle \hat{h}_{\xi}^m,\mu^n_{\lambda}(z) \rangle_{\mathcal{H}_{\mathcal{X}}} \\
    &=\bigg\langle \sum_{i=1}^n\alpha_i \psi(x_i),\sum_{j=1}^n\gamma_j(z)\psi(x_j) \bigg\rangle_{\mathcal{H}_{\mathcal{X}}} \\
    &= \sum_{i=1}^n\sum_{j=1}^n\alpha_i\gamma_j(z)k_{\mathcal{X}}(x_i,x_j) \\
    &=\alpha'K_{XX}\gamma(z) \\
    &=\alpha'w(z)
\end{align*}
where
$$
w(z):=K_{XX}\gamma(z)=K_{XX}(K_{ZZ}+n\lambda I)^{-1} K_{Zz}
$$
Note that $w$ depends on stage 1 sample matrices $X$ and $Z$ while $z$ is a test value supplied by the stage 2 sample. The regularized empirical error written in terms of dual parameter $\alpha$ is
\begin{align*}
    \hat{\mathcal{E}}^m_{\xi}(\alpha)&=\dfrac{1}{m}\sum_{i=1}^m(\tilde{y}_i-\alpha'w(\tilde{z}_i))^2+\xi\alpha'K_{XX}\alpha \\
    &=\dfrac{1}{m}\|\tilde{y}-W'\alpha\|_2^2+\xi\alpha'K_{XX}\alpha
\end{align*}
where the $i^{th}$ column of $W$ is $w(\tilde{z}_i)$. Note that $W=K_{XX}(K_{ZZ}+n\lambda I)^{-1} K_{Z\tilde{Z}}$. In this notation, $\tilde{y}$ and $\tilde{Z}$ are stage 2 sample vector and matrix. Hence
\begin{align*}
    \hat{\alpha}&= (WW'+m\xi K_{XX})^{-1}W\tilde{y}\\
    W&=K_{XX}(K_{ZZ}+n\lambda I)^{-1}K_{Z\tilde{Z}}
\end{align*}
\end{proof}

\subsubsection{Validation}\label{sec:validation}

Algorithm~\ref{alg} takes as given the values of stage 1 and stage 2 regularization parameters $(\lambda,\xi)$. Theorems~\ref{stage1} and~\ref{rate} theoretically determine optimal rates $\lambda=n^{\frac{-1}{c_1+1}}$ and $\xi=m^{-\frac{b}{bc+1}}$, respectively. For practical use, we provide a validation procedure to empirically determine values of $(\lambda,\xi)$. In some sense, the procedure implicitly estimates stage 1 prior parameter $c_1$ and stage 2 prior parameters $(b,c)$.

The procedure is as follows. Train stage 1 estimator $\mu_{\lambda}^n$ on stage 1 observations $(x_i,z_i)$ then select stage 1 regularization parameter value $\lambda^*$ to minimize out-of-sample loss, calculated from stage 2 observations $(\tilde{x}_i,\tilde{z_i})$. Train stage 2 estimator $\hat{h}_{\xi}^m$ on stage 2 observations $(\tilde{y}_i,\tilde{z}_i)$ then select stage 2 regularization parameter value $\xi^*$ to minimize out-of-sample loss, calculated from stage 1 observations $(y_i,x_i)$. Our approach assimilates the causal validation procedure of \cite{hartford2017deep} with the sample splitting inherent in KIV.

\begin{algorithm}\label{val}
Let $(x_i,y_i,z_i)$ be $n$ observations. Let $(\tilde{x}_i,\tilde{y}_i,\tilde{z}_i)$ be $m$ observations.
\begin{align*}
\gamma_{\tilde{Z}}(\lambda)&=(K_{ZZ}+n\lambda I)^{-1}K_{Z\tilde{Z}} \\
    L_1(\lambda)&=\dfrac{1}{m}tr[K_{\tilde{X}\tilde{X}}-2 K_{\tilde{X}X}\gamma_{\tilde{Z}}(\lambda)+(\gamma_{\tilde{Z}}(\lambda))'K_{XX}\gamma_{\tilde{Z}}(\lambda)] \\
    \lambda^*&=\argmin L_1(\lambda) \\
    L(\lambda,\xi)&=\dfrac{1}{n}\sum_{i=1}^n \|y_i-\hat{h}_{\xi}^m(x_i)\|^2_{\mathcal{Y}} \\
    \xi^*&= \argmin L(\lambda^*,\xi)
\end{align*}
where $\hat{h}_{\xi}^m$ is calculated by Algorithm~\ref{alg} with $\lambda=\lambda^*$.
\end{algorithm}

\begin{proof}[Proof of Algorithm~\ref{val}]
From first principles, the stage 1 out-of-sample loss is
$$
  L_1(\lambda)=\dfrac{1}{m}\sum_{i=1}^m\|\psi(\tilde{x}_i)-\mu_{\lambda}^n(\tilde{z}_i)\|^2_{\mathcal{H}_{\mathcal{X}}}
$$

Recall from the proof of Algorithm~\ref{alg}
\begin{align*}
    \mu_{\lambda}^n(z)&= \Psi_X \gamma(z) \\
    \gamma(z)&=(K_{ZZ}+n\lambda I)^{-1}K_{Zz}
\end{align*}
Therefore
\begin{align*}
    \|\psi(\tilde{x}_i)-\mu_{\lambda}^n(\tilde{z}_i)\|^2_{\mathcal{H}_{\mathcal{X}}}
    &= \|\psi(\tilde{x}_i)- \Psi_X \gamma(\tilde{z}_i)\|^2_{\mathcal{H}_{\mathcal{X}}} \\
    &=\langle\psi(\tilde{x}_i)-\Psi_X \gamma(\tilde{z}_i),\psi(\tilde{x}_i)-\Psi_X \gamma(\tilde{z}_i)  \rangle_{\mathcal{H}_{\mathcal{X}}}\\
    &=k_{\mathcal{X}}(\tilde{x}_i,\tilde{x}_i)-2K_{\tilde{x}_iX}\gamma(\tilde{z}_i)+(\gamma(\tilde{z}_i))'K_{XX}\gamma(\tilde{z}_i)
\end{align*}
\end{proof}

\subsection{Stage 1: Lemmas}

\subsubsection{Probability}

\begin{proposition}[Lemma 2 of \cite{smale2007learning}]\label{prob}
Let $\xi$ be a random variable taking values in a real separable Hilbert space $\mathcal{K}$. Suppose $\exists \tilde{M}$ s.t.
\begin{align*}
    \|\xi\|_{\mathcal{K}} &\leq \tilde{M}<\infty \quad \text{ a.s.} \\
    \sigma^2(\xi)&:=\mathbb{E}\|\xi\|_{\mathcal{K}}^2
\end{align*}
Then $\forall n\in\mathbb{N}, \forall \eta\in(0,1)$,
$$
\mathbb{P}\bigg[\bigg\|\dfrac{1}{n}\sum_{i=1}^n\xi_i-\mathbb{E}\xi\bigg\|_{\mathcal{K}}\leq\dfrac{2\tilde{M}\ln(2/\eta)}{n}+\sqrt{\dfrac{2\sigma^2(\xi)\ln(2/\eta)}{n}}\bigg]\geq 1-\eta
$$
\end{proposition}

\subsubsection{Regression}\label{sec:stage1_reg}

\begin{proposition}\label{cef}
Under Hypothesis 3
$$
E_{\rho}^*\phi(z)=\mu(z)
$$
\end{proposition}

\begin{proof}
For $h\in\mathcal{H}_{\mathcal{X}}$,
$$
\langle E_{\rho}^*\phi(z),h\rangle_{\mathcal{H}_{\mathcal{X}}}=\langle \phi(z),E_{\rho}h\rangle_{\mathcal{H}_{\mathcal{Z}}}=\langle \phi(z),\mathbb{E}_{X|Z=(\cdot)}h(X)\rangle_{\mathcal{H}_{\mathcal{Z}}}=\mathbb{E}_{X|Z=z}h(X)=\langle \mu(z),h\rangle_{\mathcal{H}_{\mathcal{X}}}
$$
The first equality is the definition of adjoint. The second holds by Proposition~\ref{well_spec_condition}. The final equality is by Proposition~\ref{prop:existenceOfConditionalMeanElement}.
\end{proof}

\begin{proposition}\label{split}
Under Hypothesis 3
$$
\mathbb{E}\|(E^*-E^*_{\rho})\phi(Z)\|^2_{\mathcal{H}_{\mathcal{X}}}=\mathcal{E}_1(E)-\mathcal{E}_1(E_{\rho})
$$
\end{proposition}

\begin{proof}
$$
\mathcal{E}_1(E)=\mathbb{E}\|\psi(X)-E^*\phi(Z)\|^2_{\mathcal{H}_{\mathcal{X}}}=\mathbb{E}\|\psi(X)-E_{\rho}^*\phi(Z)+E_{\rho}^*\phi(Z)-E^*\phi(Z)\|^2_{\mathcal{H}_{\mathcal{X}}}
$$
Expanding the square we see that the cross terms are $0$ by law of iterated expectation and Proposition~\ref{cef}.
\end{proof}

\begin{proposition}\label{lambda_min}
Under Hypotheses 3-4
$$
E_{\lambda}=\argmin_{E\in \mathcal{H}_{\Gamma}} \mathbb{E}\|(E^*-E_{\rho}^*)\phi(Z)\|^2_{\mathcal{H}_{\mathcal{X}}}+\lambda \|E\|^2_{\mathcal{H}_{\Gamma}}
$$
\end{proposition}

\begin{proof}
Corollary of Proposition~\ref{split}.
\end{proof}

\subsection{Stage 1: Theorems}\label{sec:stage1smaleconvergenceproofs}

\begin{proof}[Proof of Theorem~\ref{sol_1}]
\cite[Appendix D.1]{grunewalder2012modelling}, substituting the empirical covariance operators; or \cite[Lemma 17]{ciliberto2016consistent}.
\end{proof}

To quantify the convergence rate of $\|E^n_{\lambda}-E_{\rho}\|_{\mathcal{H}_{\Gamma}}$, we decompose it into two terms: the sampling error $\|E^n_{\lambda}-E_{\lambda}\|_{\mathcal{H}_{\Gamma}}$, and the approximation error $\|E_{\lambda}-E_{\rho}\|_{\mathcal{H}_{\Gamma}}$. To bound the sampling error, we generalize \cite[Theorem 1]{smale2007learning}.




\begin{theorem}\label{sampling}
Assume Hypotheses 2-4. $\forall \delta\in(0,1)$, the following holds w.p. $1-\delta$:
$$
\|E^n_{\lambda}-E_{\lambda}\|_{\mathcal{H}_{\Gamma}}\leq \dfrac{4\kappa(Q+\kappa \|E_{\rho}\|_{\mathcal{H}_{\Gamma}}) \ln(2/\delta)}{\sqrt{n}\lambda}
$$
\end{theorem}

\begin{proof}
Write
$$
E^n_{\lambda}-E_{\lambda}=\bigg(\mathbf{T}_1+\lambda I\bigg)^{-1}\circ\bigg(\mathbf{T}_{ZX}-\mathbf{T}_1\circ E_{\lambda}-\lambda E_{\lambda} \bigg)
$$
Observe that
\begin{align*}
    \mathbf{T}_{ZX}-\mathbf{T}_1\circ E_{\lambda}&=\dfrac{1}{n}\sum_{i=1}^n \phi(z_i)\otimes \psi (x_i) -\dfrac{1}{n}\sum_{i=1}^n [\phi(z_i)\otimes \phi(z_i)]\circ  E_{\lambda}\\
    \lambda E_{\lambda}&=T_{ZX}-T_1\circ E_{\lambda}=\int \phi(z)\otimes \psi (x)d\rho -\int \phi(z)\otimes \phi(z) d\rho \circ E_{\lambda}
\end{align*}
where the second line holds since $E_{\lambda}=(T_1+\lambda I)^{-1}\circ  T_{ZX}$ and by appealing to Proposition~\ref{op1}.

Write
$$
\xi_i=\phi(z_i)\otimes \psi (x_i)-[\phi(z_i)\otimes \phi(z_i)]\circ  E_{\lambda}=\phi(z_i)\otimes[\psi(x_i)-E_{\lambda}^* \phi(z_i)]
$$
where the second equality holds since
\begin{align*}
    \phi(z_i)\otimes \psi (x_i)-[\phi(z_i)\otimes \phi(z_i)]\circ  E_{\lambda}&=\phi(z_i)\langle \psi (x_i),\cdot\rangle_{\mathcal{H}_{\mathcal{X}}}-\phi(z_i)\langle \phi(z_i),E_{\lambda} \cdot \rangle_{\mathcal{H}_{\mathcal{Z}}} 
\end{align*}
and by the definition of the adjoint operator. 

Thus the error bound can be rewritten as
$$
E^n_{\lambda}-E_{\lambda}=\bigg(\mathbf{T}_1+\lambda I\bigg)^{-1}\circ \bigg(\dfrac{1}{n}\sum_{i=1}^n\xi_i-\mathbb{E}\xi\bigg)
$$
Observe that
\begin{align*}
    \bigg(\mathbf{T}_1+\lambda I\bigg)^{-1}&\in \mathcal{L}(\mathcal{H}_{\mathcal{Z}},\mathcal{H}_{\mathcal{Z}})\\
    \bigg(\dfrac{1}{n}\sum_{i=1}^n\xi_i-\mathbb{E}\xi\bigg) &\in \mathcal{L}_2(\mathcal{H}_{\mathcal{X}},\mathcal{H}_{\mathcal{Z}})
\end{align*}
where the latter is by Proposition~\ref{T1_lemma1}. Therefore by Propositions~\ref{opCS} and~\ref{inner},
\begin{align*}
    \|E^n_{\lambda}-E_{\lambda}\|_{\mathcal{H}_{\Gamma}}&\leq \dfrac{1}{\lambda}\Delta \\
    \Delta&=\bigg\|\dfrac{1}{n}\sum_{i=1}^n\xi_i-\mathbb{E}\xi\bigg\|_{\mathcal{H}_{\Gamma}}
\end{align*}

Note that
\begin{align*}
\|\xi_i\|_{\mathcal{H}_{\Gamma}}&\leq \kappa Q+\kappa^2\|E^*_{\lambda}\|_{\mathcal{L}_2(\mathcal{H}_{\mathcal{Z}},\mathcal{H}_{\mathcal{X}})} \\
    \sigma^2(\xi_i)&=\mathbb{E}\|\xi_i\|^2_{\mathcal{H}_{\Gamma}}\leq \kappa^2 \mathbb{E}\|\psi(X)-E_{\lambda}^*\phi(Z)\|^2_{\mathcal{H}_{\mathcal{X}}}=\kappa^2\mathcal{E}_1(E_{\lambda})
\end{align*}

By Proposition~\ref{lambda_min} with $E=0$
\begin{align*}
    \mathbb{E}\|(E_{\lambda}^*-E^*_{\rho})\phi(Z)\|^2_{\mathcal{H}_{\mathcal{X}}}+\lambda \|E_{\lambda}\|^2_{\mathcal{H}_{\Gamma}}
    &\leq \mathbb{E}\|E_\rho^*\phi(Z)\|^2_{\mathcal{H}_{\mathcal{X}}} \\
        &\leq \|E_{\rho}^* \|^2_{\mathcal{L}_2(\mathcal{H}_{\mathcal{Z}},\mathcal{H}_{\mathcal{X}})} \mathbb{E} \|\phi(Z)\|_{\mathcal{H}_{\mathcal{Z}}}^2 \\
    &\leq 
    \kappa^2\|E_{\rho}\|_{\mathcal{H}_{\Gamma}}^2 
\end{align*}
Hence
\begin{align*}
     \mathbb{E}\|(E_{\lambda}^*-E^*_{\rho})\phi(Z)\|^2_{\mathcal{H}_{\mathcal{X}}}&\leq \kappa^2\|E_{\rho}\|_{\mathcal{H}_{\Gamma}}^2  \\
     \|E^*_{\lambda}\|_{\mathcal{L}_2(\mathcal{H}_{\mathcal{Z}},\mathcal{H}_{\mathcal{X}})}=\|E_{\lambda}\|_{\mathcal{H}_{\Gamma}}&\leq \dfrac{\kappa\|E_{\rho}\|_{\mathcal{H}_{\Gamma}} }{\sqrt{\lambda}}
\end{align*}

Moreover by the definition of $E_{\rho}$ as the minimizer of $\mathcal{E}_1$,
$$
\mathcal{E}_1(E_{\rho})\leq \mathcal{E}_1(0)=\mathbb{E}\|\psi(X)\|^2_{\mathcal{H}_{\mathcal{X}}}\leq Q^2
$$
so by Proposition~\ref{split}
$$
\mathcal{E}_1(E_{\lambda})=\mathcal{E}_1(E_{\rho})+\mathbb{E}\|(E_{\lambda}^*-E^*_{\rho})\phi(Z)\|^2_{\mathcal{H}_{\mathcal{X}}}\leq Q^2+\kappa^2\|E_{\rho}\|_{\mathcal{H}_{\Gamma}}^2 
$$

In summary,
\begin{align*}
\|\xi_i\|_{\mathcal{H}_{\Gamma}}&\leq \kappa Q+\kappa^2\dfrac{ \kappa\|E_{\rho}\|_{\mathcal{H}_{\Gamma}}}{\sqrt{\lambda}}
=\kappa (Q+\kappa^2 \|E_{\rho}\|_{\mathcal{H}_{\Gamma}}/\sqrt{\lambda}) \\
    \sigma^2(\xi_i)&\leq \kappa^2 (Q^2+ \kappa^2\|E_{\rho}\|_{\mathcal{H}_{\Gamma}}^2)
\end{align*}
We then apply Proposition~\ref{prob}. With probability $1-\delta$,
$$
\Delta \leq  \kappa (Q+\kappa^2 \|E_{\rho}\|_{\mathcal{H}_{\Gamma}}/\sqrt{\lambda}) \dfrac{2\ln(2/\delta)}{n}+\sqrt{ \kappa^2 (Q^2+\kappa^2\|E_{\rho}\|_{\mathcal{H}_{\Gamma}}^2 ) \dfrac{2\ln(2/\delta)}{n}} 
$$

There are two cases.
\begin{enumerate}
    \item $\dfrac{\kappa}{\sqrt{n\lambda}}\leq \dfrac{1}{4\ln(2/\delta)}<1$.
    
    Because $a^2+b^2\leq (a+b)^2$ for $a,b\geq0$,
\begin{align*}
\Delta &< \dfrac{2\kappa Q \ln(2/\delta)}{n}+\dfrac{2\kappa^3\|E_{\rho}\|_{\mathcal{H}_{\Gamma}}\ln(2/\delta)}{n\sqrt{\lambda}}+\kappa(Q+\kappa \|E_{\rho}\|_{\mathcal{H}_{\Gamma}})\sqrt{\dfrac{2\ln(2/\delta)}{n}} \\
&=\dfrac{2\kappa Q \ln(2/\delta)}{n}
+\dfrac{2\kappa^2\|E_{\rho}\|_{\mathcal{H}_{\Gamma}}\ln(2/\delta)}{\sqrt{n}}\dfrac{\kappa}{\sqrt{n\lambda}}
+\dfrac{\kappa(Q+\kappa \|E_{\rho}\|_{\mathcal{H}_{\Gamma}}) \ln(2/\delta)}{\sqrt{n}}\sqrt{\dfrac{2}{\ln(2/\delta)}} \\
&\leq \dfrac{2\kappa Q \ln(2/\delta)}{\sqrt{n}}
+\dfrac{2\kappa^2\|E_{\rho}\|_{\mathcal{H}_{\Gamma}}\ln(2/\delta)}{\sqrt{n}}
+\dfrac{2\kappa(Q+\kappa \|E_{\rho}\|_{\mathcal{H}_{\Gamma}}) \ln(2/\delta)}{\sqrt{n}} \\
&=\dfrac{4\kappa(Q+\kappa \|E_{\rho}\|_{\mathcal{H}_{\Gamma}}) \ln(2/\delta)}{\sqrt{n}}
\end{align*}
Then recall
$$
 \|E^n_{\lambda}-E_{\lambda}\|_{\mathcal{H}_{\Gamma}}\leq \dfrac{1}{\lambda}\Delta
$$
    \item $\dfrac{\kappa}{\sqrt{n\lambda}}>\dfrac{1}{4\ln(2/\delta)}$.
    
    Observe that by the definition of $E_{\lambda}^n$
    \begin{align*}
        \dfrac{1}{n}\sum_{i=1}^n \|\psi(x_i)-(E_{\lambda}^n)^*\phi(z_i)\|^2_{\mathcal{H}_{\mathcal{X}}}+\lambda \|E_{\lambda}^n\|^2_{\mathcal{H}_{\Gamma}}
        &=\mathcal{E}_{\lambda}^n(E_\lambda^n) \\
        &\leq \mathcal{E}_{\lambda}^n(0) \\
        &=\dfrac{1}{n}\sum_{i=1}^n \|\psi(x_i)\|^2_{\mathcal{H}_{\mathcal{X}}} \\
        &\leq Q^2
    \end{align*}
    Hence
    \begin{align*}
        \|E_{\lambda}^n\|_{\mathcal{H}_{\Gamma}}\leq \dfrac{Q}{\sqrt{\lambda}}
    \end{align*}
    and
    $$
    \|E_{\lambda}^n-E_{\lambda}\|_{\mathcal{H}_{\Gamma}}\leq \dfrac{Q}{\sqrt{\lambda}}+\dfrac{\kappa\|E_{\rho}\|_{\mathcal{H}_{\Gamma}} }{\sqrt{\lambda}} 
    =\dfrac{Q+\kappa\|E_{\rho}\|_{\mathcal{H}_{\Gamma}}}{\sqrt{\lambda}}
    $$
    Finally observe that 
    $$
   \dfrac{1}{4\ln(2/\delta)} < \dfrac{\kappa}{\sqrt{n\lambda}}\iff \dfrac{Q+\kappa\|E_{\rho}\|_{\mathcal{H}_{\Gamma}}}{\sqrt{\lambda}} <\dfrac{4\kappa(Q+\kappa \|E_{\rho}\|_{\mathcal{H}_{\Gamma}}) \ln(2/\delta)}{\sqrt{n}\lambda}
    $$
\end{enumerate}
\end{proof}

To bound the approximation error, we generalize \cite[Theorem 4]{smale2005shannon}.
\begin{theorem}\label{approx}
Assume Hypotheses 2-5.
$$
\|E_{\lambda}-E_{\rho}\|_{\mathcal{H}_{\Gamma}}\leq \lambda^{\frac{c_1-1}{2}} \sqrt{\zeta_1}
$$
\end{theorem}

\begin{proof}
First observe that
$$
e^{\mathcal{Z}}_k\langle e^{\mathcal{Z}}_k,E_{\rho}\cdot \rangle_{\mathcal{H}_{\mathcal{Z}}}=e^{\mathcal{Z}}_k\langle E^*_{\rho} e^{\mathcal{Z}}_k,\cdot \rangle_{\mathcal{H}_{\mathcal{X}}}= [e^{\mathcal{Z}}_k \otimes E^*_{\rho} e^{\mathcal{Z}}_k](\cdot)
$$

By the definition of the prior, there exists a $G_1$ s.t.
$$
    G_1=T_1^{\frac{1-c_1}{2}}\circ  E_{\rho}   
     =\sum_k \nu_k^{\frac{1-c_1}{2}} e^{\mathcal{Z}}_k\langle e^{\mathcal{Z}}_k, E_{\rho} \cdot \rangle_{\mathcal{H}_{\mathcal{Z}}}  =\sum_k \nu_k^{\frac{1-c_1}{2}} e^{\mathcal{Z}}_k  \otimes [E_{\rho} ^*e^{\mathcal{Z}}_k]
$$

Hence by Proposition~\ref{inner}
$$
\|G_1\|^2_{\Gamma}=\sum_k \nu_k^{1-c_1} \|E_{\rho} ^*e^{\mathcal{Z}}_k\|^2_{\mathcal{H}_{\mathcal{X}}}
$$

By Proposition~\ref{op2}, write
\begin{align*}
    E_{\lambda}-E_{\rho}&=[(T_1+\lambda I)^{-1}\circ T_1-I]\circ E_{\rho} \\
     &=\sum_k \bigg(\dfrac{\nu_k}{\nu_k+\lambda}-1\bigg) e^{\mathcal{Z}}_k\langle e^{\mathcal{Z}}_k, E_{\rho}\cdot \rangle_{\mathcal{H}_{\mathcal{Z}}} \\
     &=\sum_k \bigg(\dfrac{\nu_k}{\nu_k+\lambda}-1\bigg) e^{\mathcal{Z}}_k  \otimes [E_{\rho} ^*e^{\mathcal{Z}}_k]
\end{align*}
Hence by Proposition~\ref{inner}
\begin{align*}
     \|E_{\lambda}-E_{\rho}\|^2_{\mathcal{H}_{\Gamma}}&=\sum_k \bigg(\dfrac{\nu_k}{\nu_k+\lambda}-1\bigg)^2 \|E_{\rho} ^*e^{\mathcal{Z}}_k\|^2_{\mathcal{H}_{\mathcal{X}}} \\
     &=\sum_k \bigg(\dfrac{\lambda}{\nu_k+\lambda}\bigg)^2 \|E_{\rho} ^*e^{\mathcal{Z}}_k\|^2_{\mathcal{H}_{\mathcal{X}}} \\
     &=\sum_k \bigg(\dfrac{\lambda}{\nu_k+\lambda}\bigg)^2 \|E_{\rho} ^*e^{\mathcal{Z}}_k\|^2_{\mathcal{H}_{\mathcal{X}}} \bigg(\dfrac{\lambda}{\lambda}\cdot \dfrac{\nu_k}{\nu_k}\cdot \dfrac{\nu_k+\lambda}{\nu_k+\lambda} \bigg)^{c_1-1}\\ 
     &=\lambda^{c_1-1} \sum_k  \nu_k^{1-c_1} \|E_{\rho} ^*e^{\mathcal{Z}}_k\|^2_{\mathcal{H}_{\mathcal{X}}}\bigg(\dfrac{\lambda}{\nu_k+\lambda}\bigg)^{3-c_1} \bigg(\dfrac{\nu_k}{\nu_k+\lambda}\bigg)^{c_1-1} \\
     &\leq \lambda^{c_1-1} \sum_k  \nu_k^{1-c_1} \|E_{\rho} ^*e^{\mathcal{Z}}_k\|^2_{\mathcal{H}_{\mathcal{X}}} \\
     &=\lambda^{c_1-1}\|G_1\|^2_{\Gamma} \\
     &\leq \lambda^{c_1-1}  \zeta_1
\end{align*}
\end{proof}

Theorems~\ref{sampling} and~\ref{approx} deliver the main stage 1 result, Theorem~\ref{stage1}, as a consequence of triangle inequality and optimizing the regularization parameter $\lambda$.

\begin{proof}[Proof of Theorem~\ref{stage1}]
By triangle inequality,
$$
    \|E^n_{\lambda}-E_{\rho}\|_{\mathcal{H}_{\Gamma}} \leq \|E^n_{\lambda}-E_{\lambda}\|_{\mathcal{H}_{\Gamma}} +\|E_{\lambda}-E_{\rho}\|_{\mathcal{H}_{\Gamma}}  \leq \dfrac{4\kappa(Q+\kappa \|E_{\rho}\|_{\mathcal{H}_{\Gamma}}) \ln(2/\delta)}{\sqrt{n}\lambda}+\lambda^{\frac{c_1-1}{2}} \sqrt{\zeta_1}$$
Minimize the RHS w.r.t. $\lambda$. Rewrite the objective as
  $$
 A\lambda^{-1} + B\lambda^{\frac{c_1-1}{2}}
 $$
then the FOC yields
$$
    \lambda =\bigg(\dfrac{2A}{B(c_1-1)}\bigg)^{\frac{2}{c_1+1}}=\bigg(\dfrac{8\kappa(Q+\kappa \|E_{\rho}\|_{\mathcal{H}_{\Gamma}}) \ln(2/\delta)}{ \sqrt{n\zeta_1}(c_1-1)}\bigg)^{\frac{2}{c_1+1}} =O(n^{\frac{-1}{c_1+1}})
    $$
Substituting this value of $\lambda$, the RHS becomes
\begin{align*}
   & A\bigg(\dfrac{2A}{B(c_1-1)}\bigg)^{-\frac{2}{c_1+1}}  + B\bigg(\dfrac{2A}{B(c_1-1)}\bigg)^{\frac{c_1-1}{c_1+1}}  \\
    &=\dfrac{B(c_1+1)}{4^{\frac{1}{c_1+1}}} \bigg(\dfrac{A}{B(c_1-1)}\bigg)^{\frac{c_1-1}{c_1+1}} \\
    &=\dfrac{ \sqrt{\zeta_1}(c_1+1)}{4^{\frac{1}{c_1+1}}} \bigg(\dfrac{4\kappa(Q+\kappa \|E_{\rho}\|_{\mathcal{H}_{\Gamma}}) \ln(2/\delta)}{ \sqrt{n\zeta_1}(c_1-1)}\bigg)^{\frac{c_1-1}{c_1+1}} 
\end{align*}
\end{proof}

\subsection{Stage 1: Corollary}




We present a corollary necessary to link stage 1 with stage 2. In doing so, we relate our work to conditional mean embedding regression.

\subsubsection{Bound}

\begin{proposition}\label{mu_pop_sol}
Assume the loss $\mathcal{E}_1^{\Xi}(\mu):=\mathbb{E}_{(X,Z)} \|\psi(X)-\mu(Z)\|^2_{\mathcal{H}_{\mathcal{X}}}$ attains a minimum on $\mathcal{H}_{\Xi}$. Then the minimizer with minimal norm $\|\cdot\|_{\mathcal{H}_{\Xi}}$ is
\begin{align*}
    \mu^-(z)&=E_{\rho}^*\phi(z) \\
     E_{\rho}^*&=T_{ZX}^*\circ T_1^{\dagger}\\
    E_{\rho}&=T_1^{\dagger}\circ  T_{ZX}
\end{align*}
where $\mu^-(z)$ is given in Definition~\ref{def:mu_minus} and $T_1^{\dagger}$ is the pseudo-inverse of $T_1$.
\end{proposition}

\begin{proof}
\cite[Lemma 16]{ciliberto2016consistent}. Note that the first equation recovers Proposition~\ref{cef}. The third equation recovers Proposition~\ref{op2}, which we know from \cite[Theorem 2]{fukumizu2004dimensionality}.
\end{proof}

\begin{proposition}[Lemma 17 of \cite{ciliberto2016consistent}]\label{mu_sample_sol}
$\forall \lambda>0$, the solution $\mu^n_{\lambda}\in\mathcal{H}_{\Xi}$ of the regularized empirical objective $\dfrac{1}{n}\sum_{i=1}^n \|\psi(x_i)-\mu(z_i)\|^2_{\mathcal{H}_{\mathcal{X}}}+\lambda\|\mu\|^2_{\mathcal{H}_{\Xi}}$ exists, is unique, and satisfies
$$
\mu^n_{\lambda}(z)=(E^n_{\lambda})^*\phi(z)
$$
\end{proposition}

\begin{corollary}\label{mu_bound_1)}
$\forall \delta\in(0,1)$, the following holds w.p. $1-\delta$: $\forall z\in\mathcal{Z}$,
$$
\|\mu^n_{\lambda}(z)-\mu^-(z)\|_{\mathcal{H}_{\mathcal{X}}}\leq r_{\mu} (\delta,n,c_1):=\kappa \cdot r_E(\delta,n,c_1)
$$
\end{corollary}

\begin{proof}
By Propositions~\ref{op2} and~\ref{mu_pop_sol}
$$
\mu^-(z)=(T_1^{\dagger}\circ T_{ZX})^* \phi(z)=(T_1^{\dagger}\circ T_1\circ  E_{\rho})^* \phi(z)=E_{\rho}^*\phi(z)
$$
so by Proposition~\ref{mu_sample_sol}
$$
\|\mu^n_{\lambda}(z)-\mu^-(z)\|_{\mathcal{H}_{\mathcal{X}}}=\|[E^{n}_{\lambda}-E_{\rho} ]^*\phi(z)\|_{\mathcal{H}_{\mathcal{X}}} \leq \|E^{n}_{\lambda}-E_{\rho}\|_{\mathcal{H}_{\Gamma}} \|\phi(z)\|_{\mathcal{H}_{\mathcal{X}}}
$$
\end{proof}


\subsubsection{Related work}\label{sec:related_mu_bound}

We relate $\mu$ to $E$ directly--an insight from \cite{grunewalder2013smooth}. In Theorem~\ref{stage1}, we generalize work by \cite{smale2005shannon,smale2007learning} to obtain a regression bound for $E$. In Corollary~\ref{mu_bound_1)}, we arrive at an RKHS-norm (and hence uniform) bound for conditional mean embedding $\mu$ that adapts to the smoothness of conditional expectation operator $E$, making use of Theorem~\ref{stage1}. The uniform bound on $\mu$ is precisely what we will need in Theorem~\ref{finite}.

Our strategy affords weaker input  assumptions and tighter bounds than the stage 1 approach of \cite{hefny2015supervised},
which uses \cite[Theorem 6]{song2009hilbert}. See Section~\ref{sec:comparisonHefnyNonparametric} for a detailed comparison.
We also make weaker assumptions than \cite[Theorem 1]{song2010nonparametric}, as detailed in Section \ref{sec:cov_review}.

Whereas Corollary~\ref{mu_bound_1)} is a bound on RKHS-norm difference $\|\mu_{\lambda}^n-\mu^{-}\|_{\mathcal{H}_{\Xi}}$, \cite[Lemma 18]{ciliberto2016consistent} contains a bound on excess risk $\mathcal{E}^{\Xi}_1(\mu_{\lambda}^n)-\mathcal{E}_1^{\Xi}(\mu^{-})$. To facilitate comparison, we translate the latter to our notation. $\forall \lambda\leq\kappa^2$ and $\delta>0$, the following holds w.p. $1-\delta$:
\begin{align*}
    \mathcal{E}^{\Xi}_1(\mu_{\lambda}^n)-\mathcal{E}_1^{\Xi}(\mu^{-})
    &=\|(E_{\lambda}^n)^*\circ S_1-R_1\|_{\mathcal{L}_2(L^2(\mathcal{Z},\rho_{\mathcal{Z}}),\mathcal{H}_{\mathcal{X}})}\\
    &\leq 4\dfrac{Q+\mathcal{A}_2^{\Xi}(\lambda)}{\sqrt{\lambda n}}\bigg(1+\sqrt{\dfrac{4\kappa^2}{\lambda\sqrt{n}}}\bigg)ln^2\dfrac{8}{\delta}+\mathcal{A}_1^{\Xi}(\lambda)
\end{align*}
where
\begin{align*}
    \mathcal{A}_1^{\Xi}(\lambda)&:=\lambda \|R_1\circ (\tilde{T}_1+\lambda)^{-1}\|_{\mathcal{L}_2(L^2(\mathcal{Z},\rho_{\mathcal{Z}}),\mathcal{H}_{\mathcal{X}})} \\
    \mathcal{A}_2^{\Xi}(\lambda)&:=\kappa \|T_{ZX}^*\circ (T_1+\lambda)^{-1}\|_{\mathcal{L}_2(\mathcal{H}_{\mathcal{Z}},\mathcal{H}_{\mathcal{X}})}
\end{align*}
Interestingly, the proof of \cite[Lemma 18]{ciliberto2016consistent} does not require Hypothesis 5, and it uses different techniques. In future work, we will leverage this result in the KIV setting and compare the consequent rates.


\subsection{Stage 2: Lemmas}

\subsubsection{Probability}
\begin{proposition}[Proposition 4 of \cite{de2005risk}]\label{prob2}
Let $\xi$ be a random variable taking values in a real separable Hilbert space $\mathcal{K}$. Suppose $\exists L,\sigma>0$ s.t.
\begin{align*}
    \|\xi\|_{\mathcal{K}} &\leq L/2 \text{ a.s} \\
    \mathbb{E}\|\xi\|_{\mathcal{K}}^2&\leq \sigma^2
\end{align*}
Then $\forall m\in\mathbb{N}, \forall \eta\in(0,1)$,
$$
\mathbb{P}\bigg[\bigg\|\dfrac{1}{m}\sum_{i=1}^m\xi_i-\mathbb{E}\xi\bigg\|_{\mathcal{K}}\leq2\bigg(\dfrac{L}{m}+\dfrac{\sigma}{\sqrt{m}}\bigg)\ln(2/\eta)\bigg]\geq 1-\eta
$$
\end{proposition}

\subsubsection{Regression}

\begin{proposition}[Lemma A.3.16 of \cite{steinwart2008support}]
The solution to the unconstrained structural operator regression problem is well-defined and satisfies
$$
H_{\rho}\mu(z)=\int_{\mathcal{Y}} yd\rho(y|\mu(z))
$$
\end{proposition}

\subsubsection{Bounds}

\begin{definition}
The residual $\mathcal{A}(\xi)$, reconstruction error $\mathcal{B}(\xi)$, and effective dimension $\mathcal{N}(\xi)$ are
\begin{align*}
    \mathcal{A}(\xi)&= \|\sqrt{T}(H_{\xi}-H_{\rho})\|^2_{\mathcal{H}_{\Omega}} \\
    \mathcal{B}(\xi)&= \|H_{\xi}-H_{\rho}\|^2_{\mathcal{H}_{\Omega}} \\
    \mathcal{N}(\xi)&= Tr[(T+\xi)^{-1}\circ T]
\end{align*}
\end{definition}

\begin{proposition}\label{first}
If $\rho\in \mathcal{P}(\zeta, b,c)$ then
\begin{align*}
 \mathcal{A}(\xi)&\leq \zeta\xi^c  \\
 \mathcal{B}(\xi)&\leq \zeta\xi^{c-1} \\
 \mathcal{N}(\xi)&\leq \beta^{1/b} \dfrac{\pi/b}{sin(\pi/b)}\xi^{-1/b}
\end{align*}
\end{proposition}

\begin{proof}
The bounds for $\mathcal{A}(\xi)$ and $\mathcal{B}(\xi)$ follow from \cite[Proposition 3]{caponnetto2007optimal} and the definition of a prior. The bound for $\mathcal{N}(\xi)$ is from \cite{sutherland2017fix}.
\end{proof}

\begin{proposition}[Theorem 2 of \cite{szabo2016learning}]
The excess error of the stage 2 estimator can be bounded by 5 terms.
$$
\mathcal{E}(\hat{H}_{\xi}^m)-\mathcal{E}(H_{\rho})\leq 5[S_{-1}+S_0+\mathcal{A}(\xi)+S_1+S_2]
$$
where
\begin{align*}
    S_{-1}&=\|\sqrt{T}\circ (\hat{\mathbf{T}}+\xi)^{-1}(\hat{\mathbf{g}}-\mathbf{g})\|^2_{\mathcal{H}_{\Omega}} \\
    S_0&= \|\sqrt{T}\circ (\hat{\mathbf{T}}+\xi)^{-1}\circ (\mathbf{T}-\hat{\mathbf{T}})H^{m}_{\xi}\|^2_{\mathcal{H}_{\Omega}} \\
    S_1&= \|\sqrt{T}\circ (\mathbf{T}+\xi)^{-1}(\mathbf{g}-\mathbf{T}H_{\rho})\|^2_{\mathcal{H}_{\Omega}}\\
    S_2&= \|\sqrt{T}\circ (\mathbf{T}+\xi)^{-1}\circ (T-\mathbf{T})(H_{\xi}-H_{\rho})\|^2_{\mathcal{H}_{\Omega}}
\end{align*}
\end{proposition}

\begin{definition}
Fix $\eta\in(0,1)$ and define the following constants
\begin{align*}
C_{\eta}&=96\ln^2(6/\eta) \\
M&=2(C+\|H_{\rho}\|_{\mathcal{H}_{\Omega}}\sqrt{B}) \\
\Sigma&=\dfrac{M}{2}
\end{align*}
\end{definition}
The choice of $C_{\eta}$ reflects a correction by \cite{sutherland2017fix} to \cite{caponnetto2007optimal}. The choices of $(M,\Sigma)$ are as in \cite[Theorem 2]{szabo2016learning}.

\begin{proposition}\label{theta}
If $m\geq \dfrac{2C_{\eta}B \mathcal{N}(\xi)}{\xi}$ and $\xi\leq \|T\|_{\mathcal{L}(\mathcal{H}_{\Omega})}$ then w.p. $1-\eta/3$
$$
\Theta(\xi):=\|(T-\mathbf{T})\circ (T+\xi)^{-1}\|_{\mathcal{L}(\mathcal{H}_{\Omega})}\leq1/2
$$
\end{proposition}

\begin{proof}
Step 2.1 of \cite[Theorem 4]{caponnetto2007optimal}.
\end{proof}

\begin{proposition}
If $m\geq \dfrac{2C_{\eta}B \mathcal{N}(\xi)}{\xi}$, $\xi\leq \|T\|_{\mathcal{L}(\mathcal{H}_{\Omega})}$, and Hypotheses 7-8 hold then w.p. $1-2\eta/3$
\begin{align*}
    S_1&\leq 32\ln^2(6/\eta) \bigg[\dfrac{BM^2}{m^2\xi}+\dfrac{\Sigma^2 \mathcal{N}(\xi)}{m}\bigg]\\
    S_2&\leq 8\ln^2(6/\eta) \bigg[\dfrac{4B^2\mathcal{B}(\xi)}{m^2\xi}+\dfrac{B\mathcal{A}(\xi)}{m\xi}\bigg]
\end{align*}
\end{proposition}

\begin{proof}
Steps 2 and 3 of \cite[Theorem 4]{caponnetto2007optimal}, appealing to Propositions~\ref{prob2} and~\ref{theta}.
\end{proof}

\begin{proposition}
$S_{-1}$ and $S_0$ may be bounded by
\begin{align*}
    S_{-1}&\leq \|\sqrt{T}\circ (\hat{\mathbf{T}}+\xi)^{-1}\|^2_{\mathcal{L}(\mathcal{H}_{\Omega})}\|\hat{\mathbf{g}}-\mathbf{g}\|^2_{\mathcal{H}_{\Omega}} \\
    S_0&\leq \|\sqrt{T}\circ (\hat{\mathbf{T}}+\xi)^{-1}\|^2_{\mathcal{L}(\mathcal{H}_{\Omega})}\|\mathbf{T}-\hat{\mathbf{T}}\|^2_{\mathcal{L}(\mathcal{H}_{\Omega})}\|H^{m}_{\xi}\|^2_{\mathcal{H}_{\Omega}}
\end{align*}
\end{proposition}

\begin{proof}
Definition of $\|\cdot\|_{\mathcal{L}(\mathcal{H}_{\Omega})}$.
\end{proof}

\begin{proposition}[Supplement 9.1 of \cite{szabo2016learning}] Suppose Hypotheses 7-8 hold.
If $m\geq \dfrac{2C_{\eta}B \mathcal{N}(\xi)}{\xi}$ and $\xi\leq \|T\|_{\mathcal{L}(\mathcal{H}_{\Omega})}$, then
\begin{align*}
    &\|H^{m}_{\xi}\|^2_{\mathcal{H}_{\Omega}} \\
    &\leq 6\bigg(\dfrac{16}{\xi}\ln^2(6/\eta)\bigg[\dfrac{M^2B}{m^2\xi}+\dfrac{\Sigma^2\mathcal{N}(\xi)}{m}\bigg]+\dfrac{4}{\xi^2}\ln^2(6/\eta)\bigg[\dfrac{4B^2\mathcal{B}(\xi)}{m^2}+\dfrac{B\mathcal{A}(\xi)}{m}\bigg]+\mathcal{B}(\xi)+\|H_{\rho}\|^2_{\mathcal{H}_{\Omega}}\bigg)
\end{align*}
\end{proposition}

\begin{proposition}[Supplement 7.1.1 and 7.1.2 of \cite{szabo2016learning}]\label{gT}
If $\|\mu^n_{\lambda}(z)-\mu^{-}(z)\|_{\mathcal{H}_{\mathcal{X}}}\leq r_{\mu}=\kappa\cdot r_E$ w.p. $1-\delta$ and Hypotheses 7-8 hold then w.p. $1-\delta$
\begin{align*}
    \|\hat{\mathbf{g}}-\mathbf{g}\|^2_{\mathcal{H}_{\Omega}}&\leq L^2C^2 r_{\mu}^{2\iota}  \\
    \|\mathbf{T}-\hat{\mathbf{T}}\|^2_{\mathcal{L}(\mathcal{H}_{\Omega})}&\leq 4BL^2 r_{\mu}^{2\iota} 
\end{align*}
\end{proposition}

\begin{proposition}\label{amgm}
$$
\|\sqrt{T}\circ (T+\xi)^{-1}\|_{\mathcal{L}(\mathcal{H}_{\Omega})}\leq\frac{1}{2\sqrt{\xi}}
$$
\end{proposition}

\begin{proof}
\cite[Step 2.1]{caponnetto2007optimal} and \cite[Supplement A.1.11]{szabo2015two} use this spectral result, which we provide for completeness. Observe that
$$
\|\sqrt{T}\circ (T+\xi)^{-1}\|_{\mathcal{L}(\mathcal{H}_{\Omega})}=\sup_{\lambda'\in \{\lambda_k\}}\frac{\sqrt{\lambda'}}{\lambda'+\xi}
$$
where $\{\lambda\}_k$ are the eigenvalues of $T$. By arithmetic-geometric mean inequality,
$$
\sqrt{\lambda'\xi}\leq \frac{\lambda'+\xi}{2}\iff \frac{\sqrt{\lambda'}}{\lambda'+\xi} \leq \frac{1}{2\sqrt{\xi}}
$$
\end{proof}

\begin{proposition}\label{last}
If $\|\mu^n_{\lambda}(z)-\mu^{-}(z)\|_{\mathcal{H}_{\mathcal{X}}}\leq r_{\mu}$ w.p. $1-\delta$, $m\geq \max\bigg\{\dfrac{2C_{\eta}B \mathcal{N}(\xi)}{\xi},\bar{m}(\delta,c_1)\bigg\}$, $\xi\leq \|T\|_{\mathcal{L}(\mathcal{H}_{\Omega})}$, and Hypotheses 7-8 hold then w.p. $1-\eta/3-\delta$
$$
\|\sqrt{T}\circ (\hat{\mathbf{T}}+\xi)^{-1}\|_{\mathcal{L}(\mathcal{H}_{\Omega})}\leq \dfrac{2}{\sqrt{\xi}}
$$
\end{proposition}

\begin{proof}
\cite[Supplement A.1.11]{szabo2015two} provides the following bound.
$$
 \|\sqrt{T}\circ (\hat{\mathbf{T}}+\xi)^{-1}\|_{\mathcal{L}(\mathcal{H}_{\Omega})}\leq \|\sqrt{T}\circ (T+\xi)^{-1}\|_{\mathcal{L}(\mathcal{H}_{\Omega})} 
 \sum_{k=0}^{\infty} \|(T-\hat{\mathbf{T}})\circ (T+\xi)^{-1}\|^k_{\mathcal{L}(\mathcal{H}_{\Omega})}
$$
Examine the RHS. By Proposition~\ref{amgm}
$$
\|\sqrt{T}\circ (T+\xi)^{-1}\|_{\mathcal{L}(\mathcal{H}_{\Omega})} \leq \frac{1}{2\sqrt{\xi}}
$$
By a telescoping argument in \cite[Supplement A.1.11]{szabo2015two}
$$
\|(T-\hat{\mathbf{T}})\circ (T+\xi)^{-1}\|_{\mathcal{L}(\mathcal{H}_{\Omega})}\leq \Theta(\xi)+\|(\mathbf{T}-\hat{\mathbf{T}})\circ (T+\xi)^{-1}\|_{\mathcal{L}(\mathcal{H}_{\Omega})}
$$
Proposition~\ref{theta} bounds the first term w.p. $1-\eta/3$. Examine the second term.
\begin{align*}
\|(\mathbf{T}-\hat{\mathbf{T}})\circ (T+\xi)^{-1}\|_{\mathcal{L}(\mathcal{H}_{\Omega})} 
&\leq \|\mathbf{T}-\hat{\mathbf{T}}\|_{\mathcal{L}(\mathcal{H}_{\Omega})}\|(T+\xi)^{-1}\|_{\mathcal{L}(\mathcal{H}_{\Omega})} \\
&\leq 2\sqrt{B}Lr^{\iota}_{\mu}  \cdot \dfrac{1}{\xi} \\
&\leq 1/4
\end{align*}
where the second inequality is by Proposition~\ref{gT} and the third inequality reflects a choice of $m$ sufficiently large. In particular, by Corollary~\ref{mu_bound_1)} it is sufficient that
$$
m\geq \bar{m}(\delta,c_1):=\bigg[\dfrac{\sqrt{\zeta_1}(c_1+1)}{4^{\frac{1}{c_1+1}}}\kappa \bigg(\dfrac{8\sqrt{B}L}{\xi}\bigg)^{1/\iota}\bigg]^{2\frac{c_1+1}{c_1-1}}\bigg[\dfrac{4\kappa(Q+\kappa \|E_{\rho}\|_{\mathcal{H}_{\Gamma}}) \ln(2/\delta)}{ \sqrt{\zeta_1}(c_1-1)}\bigg]^2
$$

Then
$$
\|(T-\hat{\mathbf{T}})\circ (T+\xi)^{-1}\|_{\mathcal{L}(\mathcal{H}_{\Omega})}\leq 1/2+1/4=3/4
$$
and hence
$$
\|\sqrt{T}\circ (\hat{\mathbf{T}}+\xi)^{-1}\|_{\mathcal{L}(\mathcal{H}_{\Omega})}\leq \dfrac{1}{2\sqrt{\xi}}\cdot \dfrac{1}{1-3/4}=\dfrac{2}{\sqrt{\xi}}
$$
\end{proof}

\subsection{Stage 2: Theorems}\label{sec:finalTheorems}

\begin{proof}[Proof of Theorem~\ref{sol_2}]
\cite[eq. 13, 14]{szabo2016learning} provide the closed form solution. Existence and uniqueness follow from \cite[Proposition 8]{cucker2002mathematical}.
\end{proof}

 To quantity the convergence rate of $\mathcal{E}(\hat{H}_{\xi}^{m})-\mathcal{E}(H_{\rho})$, we modify the central results of \cite{szabo2016learning}, replacing their first stage convergence argument with our own derived above.

\begin{theorem}\label{finite}
Assume Hypotheses 1-9. If $m$ is large enough and $\xi\leq \|T\|_{\mathcal{L}(\mathcal{H}_{\Omega})}$ then $\forall \delta\in(0,1)$ and $\forall \eta\in(0,1)$, the following holds w.p. $1-\eta-\delta$:
\begin{align*}
&\mathcal{E}(\hat{H}_{\xi}^{m})-\mathcal{E}(H_{\rho})
\leq r_H(\delta,n,c_1;\eta,m,b,c)
:= 5\bigg\{\dfrac{4}{\xi}\cdot L^2C^2(\kappa\cdot r_{E})^{2\iota} + \dfrac{4}{\xi}\cdot  4 BL^2 (\kappa\cdot r_{E})^{2\iota}\\
  &\quad \cdot 6\bigg(\dfrac{16}{\xi}\ln^2(6/\eta)\bigg[\dfrac{M^2B}{m^2\xi}+\dfrac{\Sigma^2 }{m}\beta^{1/b} \dfrac{\pi/b}{sin(\pi/b)}\xi^{-1/b}\bigg]\\
  &\quad +\dfrac{4}{\xi^2}\ln^2(6/\eta)\bigg[\dfrac{4B^2 \zeta\xi^{c-1}}{m^2}+\dfrac{B\zeta \xi^c }{m}\bigg]+ \zeta\xi^{c-1}+\|H_{\rho}\|^2_{\mathcal{H}_{\Omega}}\bigg) \\
    &\quad + \zeta \xi^c
   +32\ln^2(6/\eta) \bigg[\dfrac{BM^2}{m^2\xi}+\dfrac{\Sigma^2}{m}\beta^{1/b} \dfrac{\pi/b}{sin(\pi/b)}\xi^{-1/b}\bigg]+8\ln^2(6/\eta) \bigg[\dfrac{4B^2\zeta\xi^{c-1}}{m^2\xi}+\dfrac{B\zeta\xi^c}{m\xi}\bigg]
   \bigg\} 
\end{align*}
\end{theorem}
Note that the convergence rate is calibrated by $c_1$, the smoothness of the conditional expectation operator $E_{\rho}$; $c$, the smoothness of the structural operator $H_{\rho}$; and $b$, the effective input dimension.

\begin{proof}
By Propositions~\ref{first} to~\ref{last},
\begin{align*}
    \mathcal{E}(\hat{H}_{\xi}^m)-\mathcal{E}(H_{\rho})&\leq 5[S_{-1}+S_0+\mathcal{A}(\xi)+S_1+S_2] \\
    S_{-1}
    &\leq \dfrac{4}{\xi}\cdot L^2C^2r^{2\iota}_{\mu}  \\
    S_{0}
    &\leq  \dfrac{4}{\xi}\cdot  4 BL^2r^{2\iota}_{\mu} \cdot \|H_{\xi}^m\|^2_{\mathcal{H}_{\Omega}}
  \\
    \|H^{m}_{\xi}\|^2_{\mathcal{H}_{\Omega}}
    &\leq 6\bigg(\dfrac{16}{\xi}\ln^2(6/\eta)\bigg[\dfrac{M^2B}{m^2\xi}+\dfrac{\Sigma^2 }{m}\beta^{1/b} \dfrac{\pi/b}{sin(\pi/b)}\xi^{-1/b}\bigg] \\
    &\quad +\dfrac{4}{\xi^2}\ln^2(6/\eta)\bigg[\dfrac{4B^2 \zeta\xi^{c-1}}{m^2}+\dfrac{B\zeta \xi^c }{m}\bigg]+ \zeta\xi^{c-1}+\|H_{\rho}\|^2_{\mathcal{H}_{\Omega}}\bigg) \\
    \mathcal{A}(\xi)
    &\leq \zeta \xi^c \\
    S_1
    &\leq  32\ln^2(6/\eta) \bigg[\dfrac{BM^2}{m^2\xi}+\dfrac{\Sigma^2}{m}\beta^{1/b} \dfrac{\pi/b}{sin(\pi/b)}\xi^{-1/b}\bigg] 
    \\
    S_2
    &\leq  8\ln^2(6/\eta) \bigg[\dfrac{4B^2\zeta\xi^{c-1}}{m^2\xi}+\dfrac{B\zeta\xi^c}{m\xi}\bigg]
\end{align*}
Finally use Corollary~\ref{mu_bound_1)} to write $r_{\mu}=\kappa\cdot r_{E}$.
\end{proof}

\begin{proof}[Proof of Theorem~\ref{rate}]
Ignoring constants in Theorem~\ref{finite} yields
\begin{align*}
    S_{-1}
    &=O\bigg(\dfrac{r^{2\iota}_{\mu}}{\xi}\bigg)  \\
    S_{0}
    &=O\bigg(\dfrac{r^{2\iota}_{\mu}}{\xi}\cdot \|H_{\xi}^m\|^2_{\mathcal{H}_{\Omega}}\bigg)  \\
    \|H^{m}_{\xi}\|^2_{\mathcal{H}_{\Omega}}
    &=O\bigg(\dfrac{1}{m^2\xi^2}+\dfrac{1}{m\xi^{1+1/b}}+\dfrac{1}{m^2\xi^{3-c}}+\dfrac{1}{m\xi^{2-c}}+\xi^{c-1}+1\bigg) \\
    \mathcal{A}(\xi)&=O(\xi^c)\\
    S_1
    &=O\bigg(\dfrac{1}{m^2\xi}+\dfrac{1}{m\xi^{1/b}}\bigg)\\
    S_2
    &=O\bigg(\dfrac{1}{m^2\xi^{2-c}}+\dfrac{\xi^{c-1}}{m}\bigg)
\end{align*}
The last term in the bound on $\|H^{m}_{\xi}\|^2_{\mathcal{H}_{\Omega}}$ implies that the bounding terms of $S_0$ dominate those of $S_{-1}$. Within the terms bounding $\|H^{m}_{\xi}\|^2_{\mathcal{H}_{\Omega}}$, observe that $\frac{1}{m^2\xi^2}$ dominates $\frac{1}{m^2\xi^{3-c}}$; $\frac{1}{m\xi^{1+1/b}}$ dominates $\frac{1}{m\xi^{2-c}}$; and $1$ dominates $\xi^{c-1}$. These statements follow from the restrictions $b>1$ and $c\in(1,2]$ in the definition of a prior as well as $\xi\rightarrow0$. Likewise, the terms bounding $S_1$ dominate the terms bounding $S_2$. In summary, we arrive at a statement analogous to \cite[eq. 19]{szabo2016learning}.
\begin{align*}
    &\mathcal{E}(\hat{H}_{\xi}^m)-\mathcal{E}(H_{\rho})=O\bigg(\dfrac{r^{2\iota}_{\mu}}{\xi}
 \cdot \bigg[\dfrac{1}{m^2\xi^2}+\dfrac{1}{m\xi^{1+1/b}}+1\bigg]
 +\xi^c+\dfrac{1}{m^2\xi}+\dfrac{1}{m\xi^{1/b}}\bigg) \\
 &\quad \text{s.t. } m\xi^{1+1/b}\geq1, r^{2\iota}_{\mu}\leq \xi^2
\end{align*}
By Corollary~\ref{mu_bound_1)}, Theorem~\ref{stage1}, and the choices of $\lambda$ and $n$ in the statement of Theorem~\ref{rate}
$$
r^{2\iota}_{\mu}=O\bigg( [(n^{-\frac{1}{2}})^{\frac{c_1-1}{c_1+1}}]^{2\iota}\bigg)=O(m^{-a})
$$
With this substitution, we arrive at a statement analogous to \cite[eq. 20]{szabo2016learning}.
\begin{align*}
    &\mathcal{E}(\hat{H}_{\xi}^{m})-\mathcal{E}(H_{\rho})
 =O\bigg(\dfrac{1}{m^{2+a}\xi^3}+\dfrac{1}{m^{1+a}\xi^{2+1/b}}+\dfrac{1}{m^a\xi}
 +\xi^c+\dfrac{1}{m^2\xi}+\dfrac{1}{m\xi^{1/b}}\bigg) \\
 &\quad \text{s.t. } m\xi^{1+1/b}\geq1,  m^a\xi^2\geq1
\end{align*}
The final result is \cite[Theorem 5]{szabo2016learning}.
\end{proof}

\subsection{Experiments}\label{sec:sim}

\subsubsection{Designs}

\begin{figure}[H]
  \begin{subfigure}[b]{0.5\textwidth}
    \includegraphics[width=0.9\textwidth]{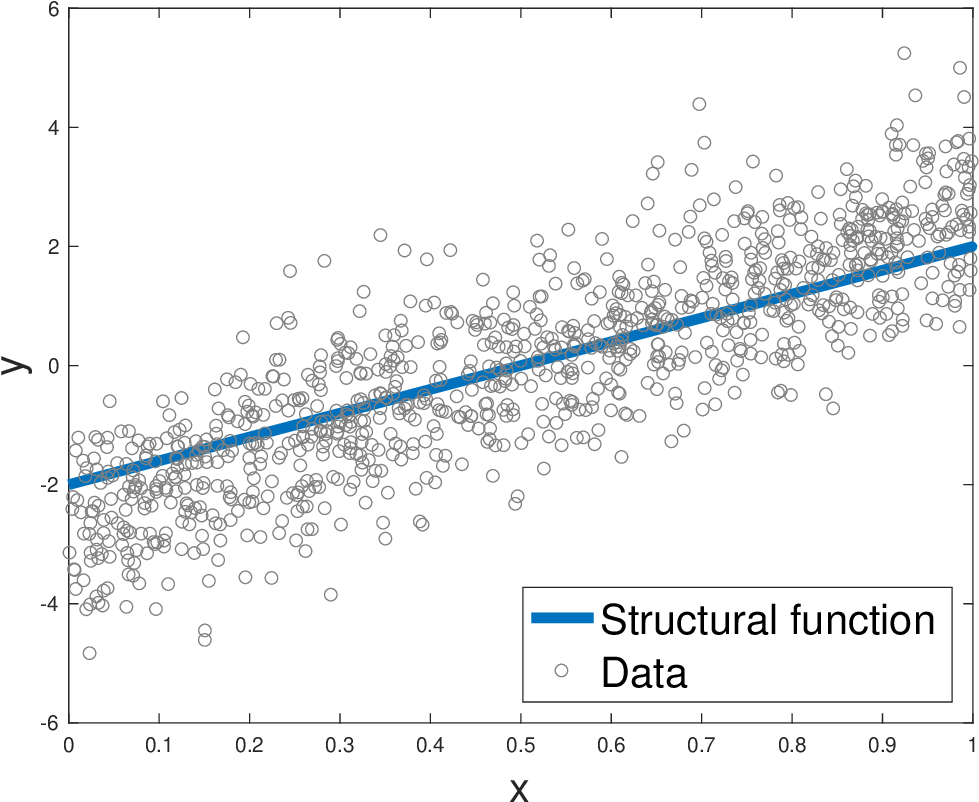}
    \cprotect\caption{Linear}
    \label{linear_dgp}
  \end{subfigure}
  \begin{subfigure}[b]{0.5\textwidth}
    \includegraphics[width=0.9\textwidth]{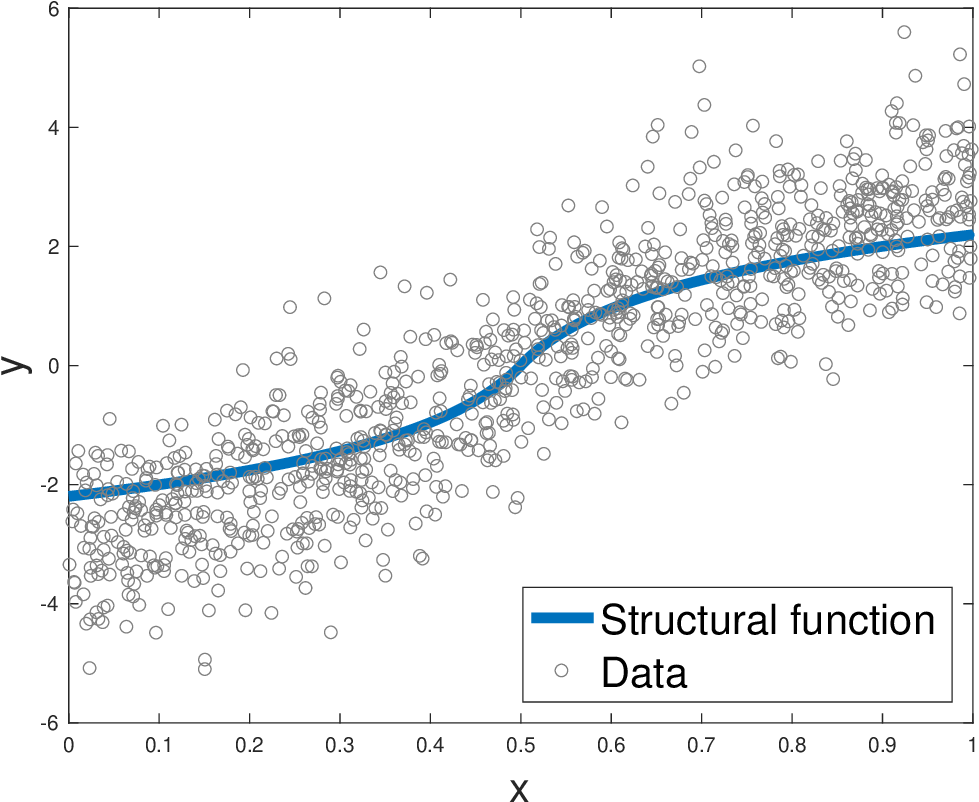}
    \cprotect\caption{Sigmoid}
    \label{sigmoid_dgp}
  \end{subfigure}
  \cprotect\caption{Linear and sigmoid data generating processes. Training sample size is $n+m=1000$}
  \label{uni_dgp}
  \vspace{-0pt}
\end{figure}

The linear and sigmoid simulation designs are from \cite{chen2018optimal}, adapted from \cite{newey2003instrumental}. One simulation consists of a sample of $n+m\in\{1000,5000,10000\}$ observations. A given observation is generated from the IV model
\begin{align*}
    Y&=h(X)+e,\quad \mathbb{E}[e|Z]=0
\end{align*}
where $Y$ is the output, $X$ is the input, $Z$ is the instrument, and $e$ is confounding noise. In particular, for the linear design
\begin{align*}
    h(x)&=4x-2
\end{align*}
while for the sigmoid design
\begin{align*}
    h(x)&=\ln (|16x-8|+1)\cdot sgn(x-0.5)
\end{align*}
Data are sampled as
\begin{align*}
\begin{pmatrix}
e \\ V \\ W
\end{pmatrix}&
\overset{i.i.d.}{\sim}N\left(\begin{pmatrix}
0 \\ 0 \\ 0
\end{pmatrix},\begin{pmatrix}
1 & \frac{1}{2} & 0 \\
\frac{1}{2} & 1 & 0 \\
0 & 0 & 1
\end{pmatrix}\right)\\
X&=\Phi\left(\frac{W+V}{\sqrt{2}}\right)\\
Z&=\Phi(W)
\end{align*}

We visualize 1 simulation, consisting of $n+m=1000$ observations, in Figure~\ref{uni_dgp}. The blue curve illustrates the structural function $h$. Grey dots depict noisy observations. The noise $e$ has positively sloped bias relative to the structural function $h$. From observations of $(Y,X,Z)$, we estimate $\hat{h}$ by several methods. For each estimated $\hat{h}$, we measure out-of-sample error as the mean square error of $\hat{h}$ versus true $h$ applied to 1000 evenly spaced values $x\in[0,1]$. We report $log_{10}(MSE)$.

The demand simulation design is from \cite{hartford2017deep}. One simulation consists of a sample of $n+m\in\{1000,5000,10000\}$ observations. A given observation is generated from the IV model
\begin{align*}
    Y&=h(X)+e,\quad \mathbb{E}[e|Z]=0
\end{align*}
where $Y$ is the output, $X=(P,T,S)$ are inputs, and $Z=(C,T,S)$ are instruments. Recall that $Y$ is sales, $P$ is the endogenous input instrumented by supply cost-shifter $C$, and $(T,S)$ are exogenous inputs interpretable as time of year and customer sentiment. While $(P,T,C)$ are continuous random variables, $S$ is discrete--a novel feature of this design. $e$ is confounding noise. 

\begin{align*}
    &h(p,t,s)=100+(10+p)s\psi(t)-2p \\
    &\psi(t)=2\left[
    \frac{(t-5)^4}{600}
    +\exp\left(-4(t-5)^2\right)+\frac{t}{10}-2
    \right]
\end{align*}

\begin{wrapfigure}{R}{\textwidth/4}
\vspace{-15pt}
  \begin{center}
    \includegraphics[width=\textwidth/4]{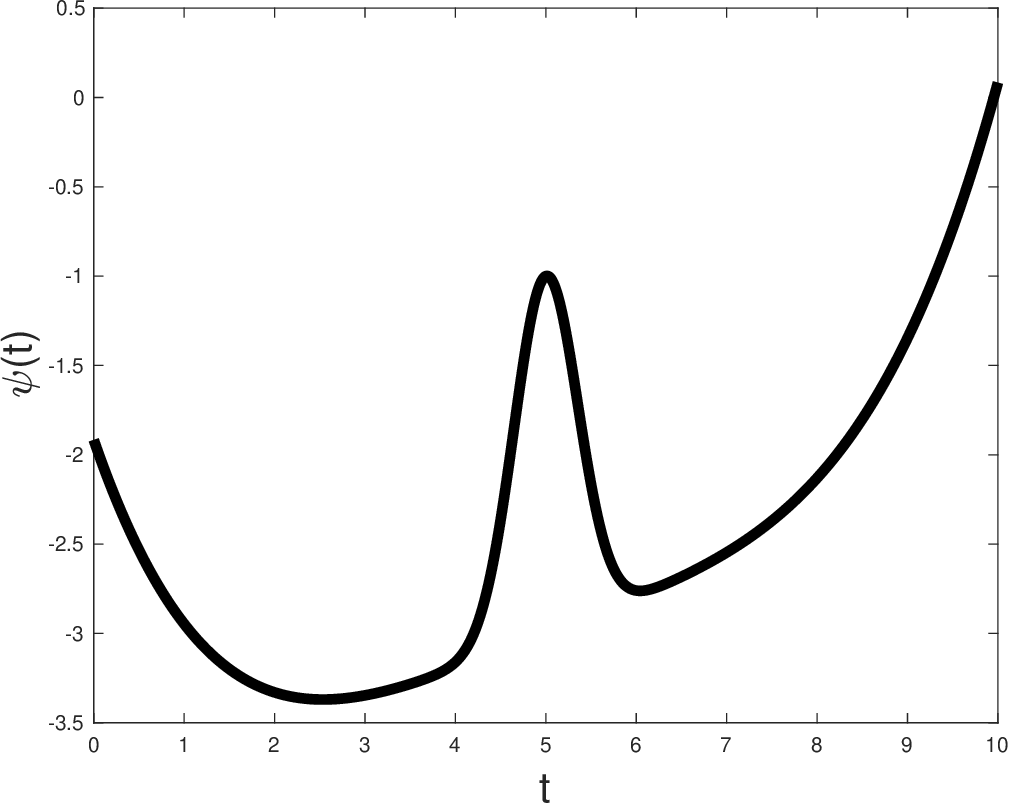}
     \vspace{-10pt}
    \caption{Demand nonlinearity $\psi(t)$}
    \label{demand_dgp}
    \end{center}
        \vspace{-90pt}
\end{wrapfigure}

Data are sampled as
\begin{align*}
S&\overset{i.i.d.}{\sim} Unif\{1,...,7\}\\
T&\overset{i.i.d.}{\sim}Unif[0,10]\\
\begin{pmatrix}
C \\ V 
\end{pmatrix}&
\overset{i.i.d.}{\sim}N\left(\begin{pmatrix}
0 \\ 0 
\end{pmatrix},\begin{pmatrix}
1 & 0 \\
0 & 1
\end{pmatrix}\right)\\
e&\overset{i.i.d.}{\sim}N(\rho V,1-\rho^2)  \\
P&=25+(C+3)\psi(T)+V
\end{align*}

From observations of $(Y,P,T,S,C)$, we estimate $\hat{h}$ by several methods. For each estimated $\hat{h}$, we measure out-of-sample error as the mean square error of $\hat{h}$ versus true $h$ applied to 2800 values of $(p,t,s)$. Specifically, we consider 20 evenly spaced values of $p\in[10,25]$, 20 evenly spaced values of $t\in[0,10]$, and all 7 values $s\in\{1,...,7\}$ (In an earlier version, we considered $p\in[2.5,14.5]$, which is less representative.) We report $log_{10}(MSE)$.

\newpage

\subsubsection{Algorithms}

\begin{wrapfigure}{R}{\textwidth/2}
\vspace{-10pt}
  \begin{center}
    \includegraphics[width=\textwidth/2]{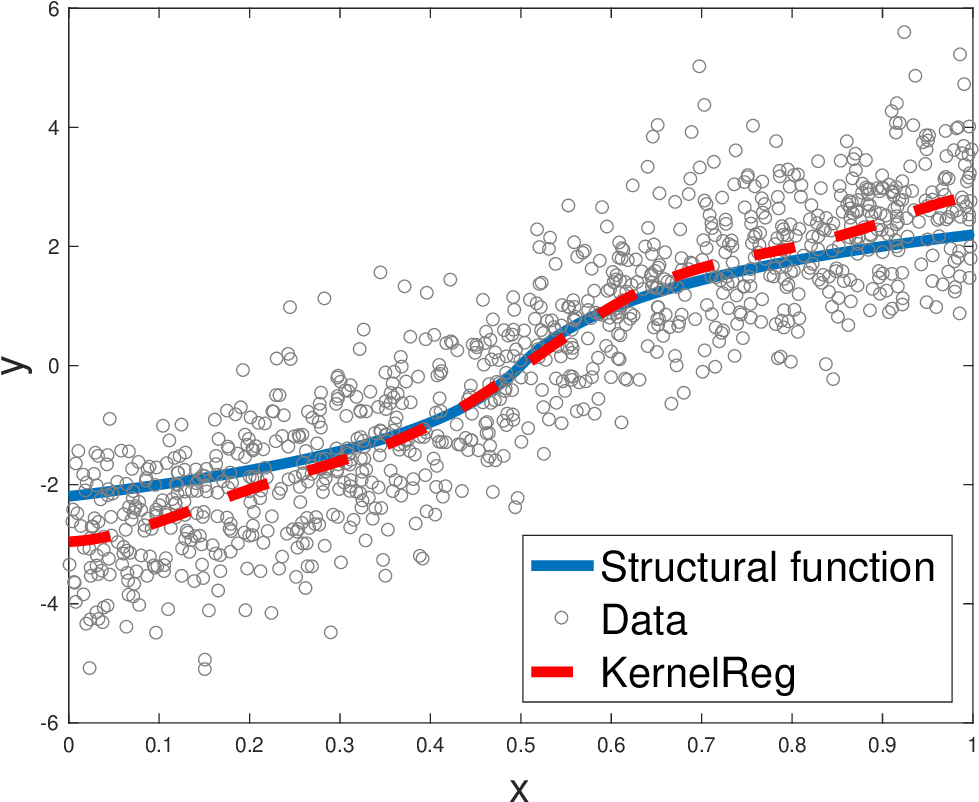}
    \vspace{-10pt}
    \cprotect\caption{\verb|KernelReg| on the sigmoid design}
    \vspace{-10pt}
    \label{KernelReg}
    \end{center}
\end{wrapfigure}

\verb|KernelReg|. We implement kernel ridge regression using Gaussian kernel $k_{\mathcal{X}}$. We set the kernel hyperparameter--the lengthscale--equal to the median interpoint distance of inputs, a standard practice. When inputs are multidimensional as in the demand design, we use the kernel obtained as the product of scalar kernels for each input dimension. Each lengthscale is set according to the median interpoint distance for that input dimension. We tune the Tikhonov regularization parameter by cross-validation with two folds. Figure~\ref{KernelReg} visualizes the performance of \verb|KernelReg| on the sigmoid design with $n+m=1000$. Kernel ridge regression ignores the instrument $Z$, and it is biased away from the structural function due to confounding noise. The remaining algorithms make use of instrument $Z$ to overcome this issue.

\verb|SieveIV|. We implement sieve IV with sample splitting using $B$-spline basis. We set the basis hyperparameters according to the preferred specification of \cite{chen2018optimal}: $4^{th}$ order polynomial with 1 interior knot. We implement sieve IV without Tikhonov regularization (as originally formulated), and with Tikhonov regularization. We tune Tikhonov regularization parameters $(\lambda,\xi)$ according to Algorithm~\ref{val}. Figure~\ref{SieveIV} visualizes the performance of \verb|SieveIV| on the sigmoid design with $n+m=1000$. Tikhonov regularization dramatically improves performance in both the sigmoid and demand designs. There is still room for improvement, however, since \verb|SieveIV| is constrained to finite dictionaries of basis functions.

\verb|SmoothIV|. We implement Nadaraya-Watson IV using the \verb|R| command \verb|npregiv|. We set the regularization option to Tikhonov, in order to implement the estimator of \cite{darolles2011nonparametric}. Otherwise we maintain default options. As in \cite{hartford2017deep}, we only apply this estimator to training samples of size $n+m=1000$ due to its lengthy running time. Figure~\ref{SmoothIV} visualizes the performance of \verb|SmoothIV| on the sigmoid design with $n+m=1000$. \verb|SmoothIV| is clearly an improvement on its predecessor, the original \verb|SieveIV|. By imposing Tikhonov regularization in stage 2, the algorithm greatly reduces variance. The Nadaraya-Watson style stage 1 estimator appears to be the reason why \verb|SmoothIV| fails to learn the structural function's sigmoid shape. Overfitting in stage 1 could explain why the final estimate has more inflection points than the true structural function. 

\begin{figure}[h]
 \vspace{-0pt}
  \begin{subfigure}[b]{0.5\textwidth}
    \includegraphics[width=0.9\textwidth]{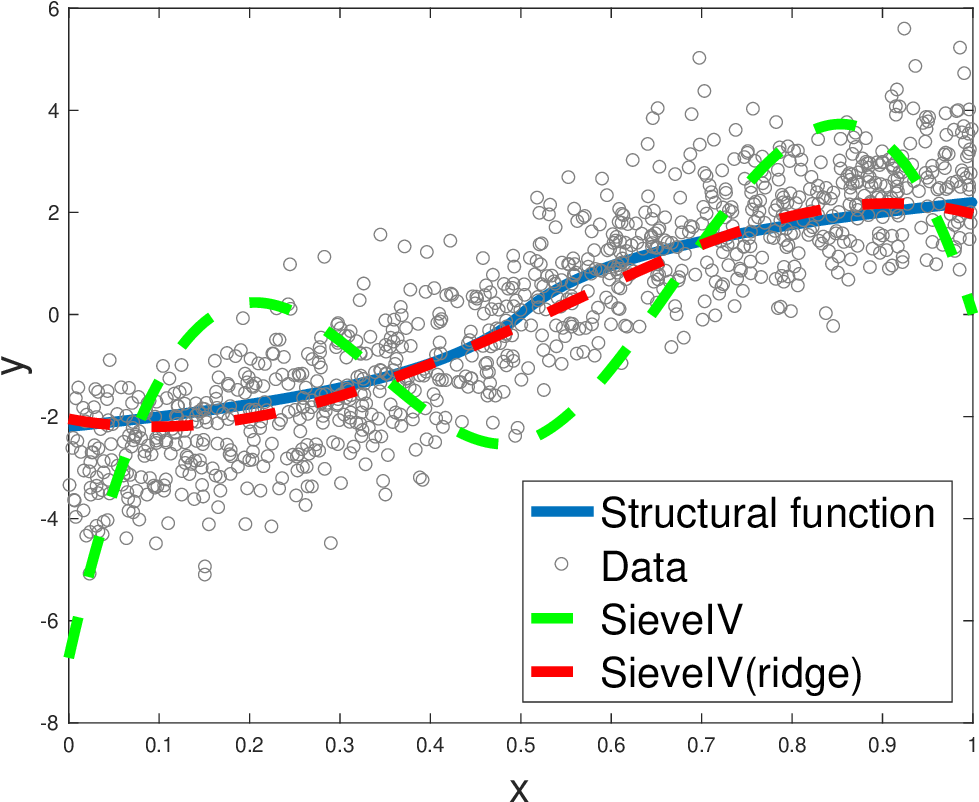}
    \cprotect\caption{\verb|SieveIV|}
    \label{SieveIV}
  \end{subfigure}
  \begin{subfigure}[b]{0.5\textwidth}
    \includegraphics[width=0.9\textwidth]{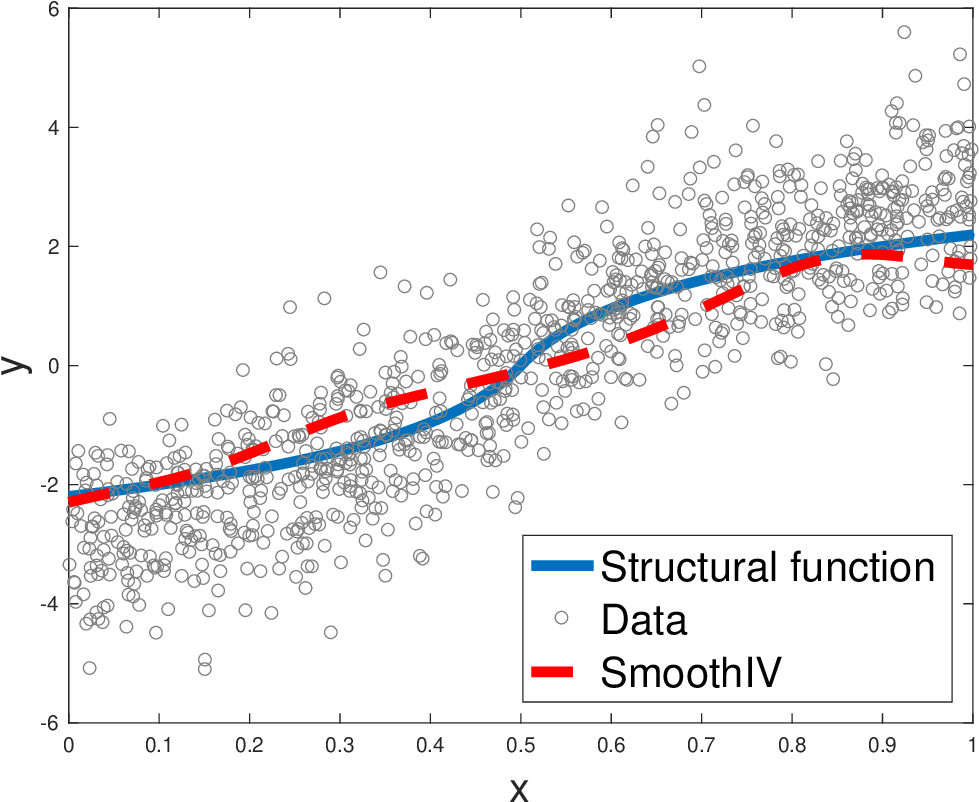}
    \cprotect\caption{\verb|SmoothIV|}
    \label{SmoothIV}
  \end{subfigure}
    \cprotect\caption{\verb|SieveIV| and \verb|SmoothIV| on the sigmoid design}
    \vspace{-0pt}
\end{figure}

\verb|DeepIV|. We implement deep IV with sample splitting using the \verb|python| software accompanying the paper by \cite{hartford2017deep}. We implement deep IV with and without biased gradients in the training optimization. Figure~\ref{DeepIV} visualizes the performance of \verb|DeepIV| on the sigmoid design with $n+m=1000$. In both the sigmoid and demand designs, unbiased gradients lead to better performance. Biased gradients improve performance in a high-dimensional MNIST design that we do not implement here. Like other neural network models, \verb|DeepIV| requires a relatively large training sample size to achieve reliable performance on simple tasks like learning a smooth curve.

\verb|KernelIV|. We implement KIV with sample splitting using Gaussian kernels $k_{\mathcal{X}}$ and $k_{\mathcal{Z}}$. We set lengthscales according to median interpoint distance as described for \verb|KernelReg|. When inputs are multimensional, we use the product of scalar kernels as described for \verb|KernelReg|. We tune Tikhonov regularization parameters $(\lambda,\xi)$ according to Algorithm~\ref{val}. Figure~\ref{KernelIV} visualizes the performance of \verb|KernelIV| on the sigmoid design with $n+m=1000$.

\begin{figure}[H]
  \begin{subfigure}[b]{0.5\textwidth}
    \includegraphics[width=0.9\textwidth]{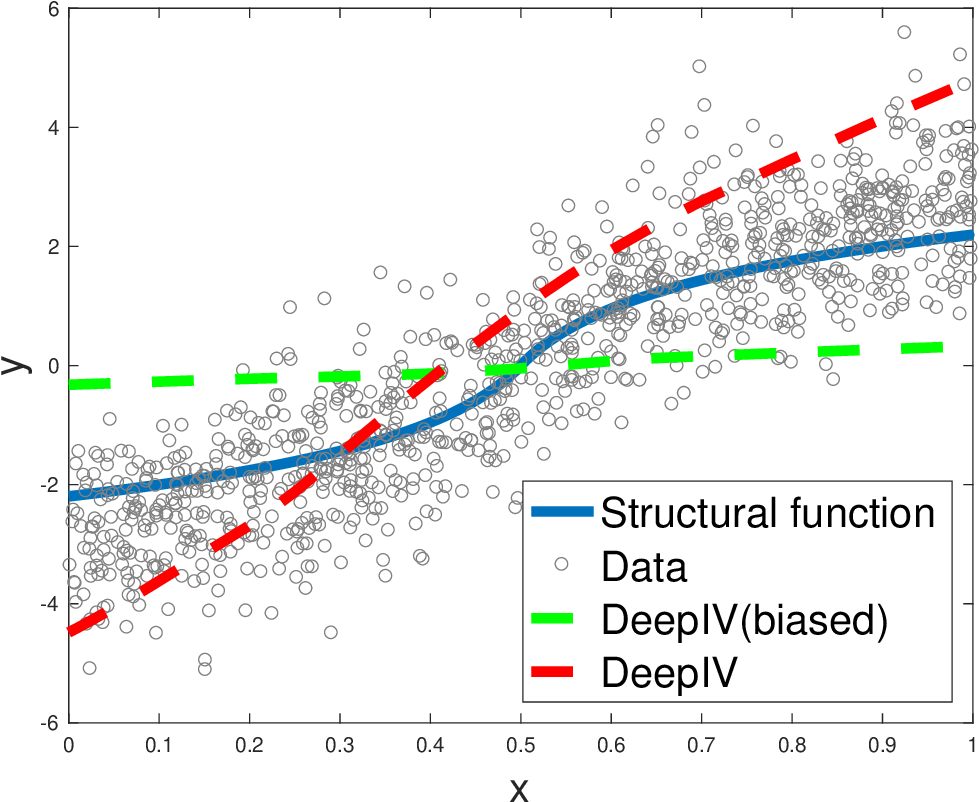}
    \cprotect\caption{\verb|DeepIV|}
    \label{DeepIV}
  \end{subfigure}
  \begin{subfigure}[b]{0.5\textwidth}
    \includegraphics[width=0.9\textwidth]{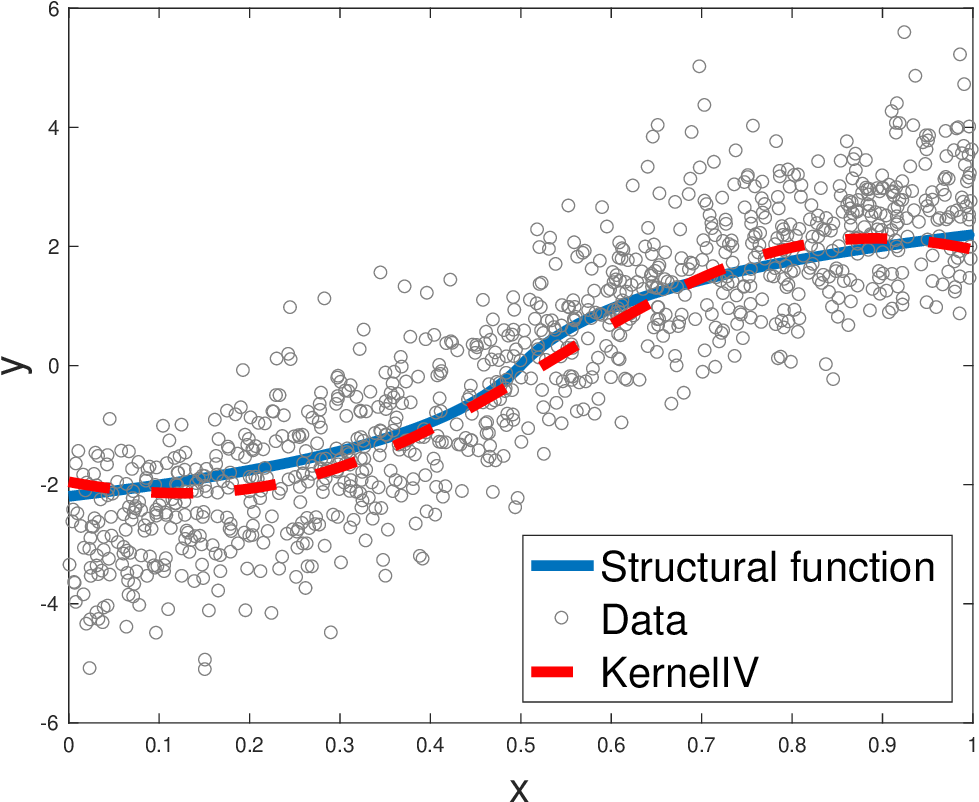}
    \cprotect\caption{\verb|KernelIV|}
    \label{KernelIV}
  \end{subfigure}
  \cprotect\caption{\verb|DeepIV| and \verb|KernelIV| on the sigmoid design}
  \vspace{-10pt}
\end{figure}
\subsubsection{Results}~\label{further_sim}

\begin{figure}[H]
 \vspace{-15pt}
  \begin{subfigure}[b]{0.5\textwidth}
    \includegraphics[width=\textwidth]{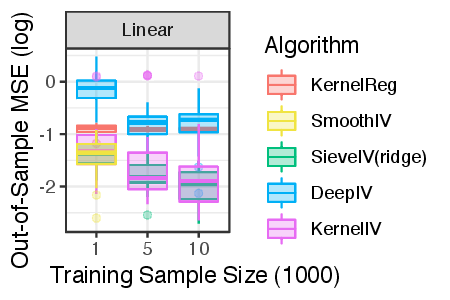}
      \vspace{-15pt}
    \cprotect\caption{Linear}
    \label{linear}
  \end{subfigure}
  \begin{subfigure}[b]{0.5\textwidth}
    \includegraphics[width=\textwidth]{sigmoid.eps}
      \vspace{-15pt}
    \cprotect\caption{Sigmoid}
    \label{sigmoid}
  \end{subfigure}
  \cprotect\caption{Linear and sigmoid designs}
  \vspace{-10pt}
  \label{uni}
\end{figure}

\begin{wrapfigure}{R}{\textwidth/2}
\vspace{-0pt}
  \begin{center}
    \includegraphics[width=\textwidth/2]{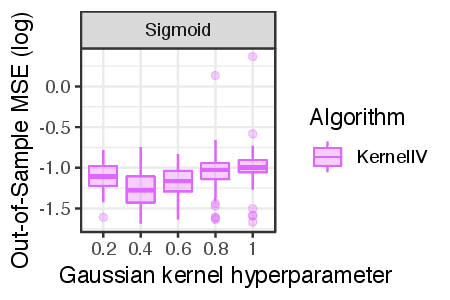}
     \vspace{-15pt}
    \caption{Robustness study}
     \vspace{-10pt}
    \label{robust}
    \end{center}
\end{wrapfigure}

For each algorithm, design, and sample size, we implement 40 simulations and calculate MSE with respect to the true structural function $h$. Figures~\ref{demand} and~\ref{uni} visualize results. In the linear design, \verb|KernelIV| performs about as well as \verb|SieveIV| improved with Tikhonov regularization. Intuitively, in the linear design the true structural function $h$ is finite-dimensional, and the method that uses a finite dictionary of basis functions (\verb|SieveIV|) displays less variability across simulations when training sample sizes are small. Insofar as \verb|SieveIV| is a special case of \verb|KernelIV|, one could interpret this outcome as reflecting a more appropriate choice of kernel. In the sigmoid design, \verb|KernelIV| performs best across sample sizes. In the demand design, \verb|SmoothIV| performs best for sample size $n+m=1000$. Among estimators that we are able to implement, \verb|KernelIV| performs best for sample sizes $n+m=5000$ and $n+m=10000$.

\newpage

Finally, we conduct a robustness study to evaluate the sensitivity of \verb|KernelIV| to hyperparameter tuning. We apply \verb|KernelIV| to the sigmoid design with $n+m=1000$, varying the lengthscale for Guassian kernel $k_{\mathcal{X}}$. For each lengthscale value in $\{0.2, 0.4, 0.6, 0.8, 1.0\}$, we implement 40 simulations and calculate MSE with respect to the true structural function $h$. For comparison, the median interpoint distance rule sets lengthscale to $0.3$. Figure~\ref{robust} visualizes results: alternative lengthscale values depreciate performance of \verb|KernelIV|, but \verb|KernelIV| still outperforms its competitors in Figure~\ref{sigmoid}. We recommend that practitioners use the median interpoint distance rule.

\end{document}